\documentclass[11pt]{article}
\usepackage[margin=1in]{geometry}
\usepackage[numbers, compress]{natbib}

\usepackage[utf8]{inputenc} 
\usepackage[T1]{fontenc}    
\usepackage{hyperref}       
\usepackage{url}            
\usepackage{booktabs}       
\usepackage{amsfonts}       
\usepackage{nicefrac}       
\usepackage{microtype}      
\usepackage{xcolor}         
\usepackage{enumitem}
\usepackage{setspace}

\usepackage{yub}

\hypersetup{
    colorlinks,
    linkcolor={blue!50!black},
    citecolor={blue!50!black},
}
\colorlet{linkequation}{blue}

\mathchardef\mhyphen="2D
\renewcommand{\setto}{\leftarrow}
\newcommand{\up}[1]{\overline{#1}}
\newcommand{\low}[1]{\underline{#1}}

\newcommand{\Bonus}{\textsc{Bonus}}
\newcommand{\rf}{{\rm ref}}

\newcommand{\Unif}{{\rm Unif}}
\newcommand{\ba}{\bm{a}}
\newcommand{\bzero}{\bm{0}}
\renewcommand{\epsilon}{\eps}
\renewcommand{\hat}{\what}

\newcommand{\TV}{{\rm TV}}
\newcommand{\KL}{{\rm KL}}

\def\vilcb{\textsc{VI-LCB}}
\def\ucbviuplow{\textsc{UCBVI-UpLow}}
\def\truncpevi{\textsc{Truncated-PEVI-ADV}}
\def\pevi{\textsc{PEVI-Adv}}
\def\peviadv{\textsc{PEVI-Adv}}
\def\hoovi{\textsc{HOOVI}}

\def\shownotes{1}  
\ifnum\shownotes=1
\newcommand{\authnote}[2]{{\scriptsize $\ll$\textsf{#1 notes: #2}$\gg$}}
\else
\newcommand{\authnote}[2]{}
\fi

\usepackage{mathtools}
\mathtoolsset{showonlyrefs,showmanualtags}

\usepackage{dsfont}
\newcommand{\1}{\mathds{1}}
\renewcommand{\indic}[1]{\mathds{1}\left\{#1\right\}}

\newcommand{\Eqref}[1]{Eq.\eqref{#1}}

\allowdisplaybreaks

\title{Policy Finetuning: Bridging Sample-Efficient\\ Offline and Online Reinforcement Learning}

\author{
    Tengyang Xie\thanks{University of Illinois at Urbana-Champaign. E-mail: \texttt{\{tx10, nanjiang\}@illinois.edu}}
    \and
    Nan Jiang\footnotemark[1]
    \and
    Huan Wang\thanks{Salesforce Research. E-mail: \texttt{\{huan.wang, cxiong, yu.bai\}@salesforce.com}}
    \and
    Caiming Xiong\footnotemark[2]
    \and
    Yu Bai\footnotemark[2]
}

\begin{document}

\maketitle

\begin{abstract}


Recent theoretical work studies sample-efficient reinforcement learning (RL) extensively in two settings: learning interactively in the environment (online RL), or learning from an offline dataset (offline RL). However, existing algorithms and theories for learning near-optimal policies in these two settings are rather different and disconnected. Towards bridging this gap, this paper initiates the theoretical study of \emph{policy finetuning}, that is, online RL where the learner has additional access to a ``reference policy'' $\mu$ close to the optimal policy $\pi_\star$ in a certain sense. We consider the policy finetuning problem in episodic Markov Decision Processes (MDPs) with $S$ states, $A$ actions, and horizon length $H$.
We first design a sharp \emph{offline reduction} algorithm---which simply executes $\mu$ and runs offline policy optimization on the collected dataset---that finds an $\varepsilon$ near-optimal policy within $\widetilde{O}(H^3SC^\star/\varepsilon^2)$ episodes, where $C^\star$ is the single-policy concentrability coefficient between $\mu$ and $\pi_\star$. This offline result is the first that matches the sample complexity lower bound in this setting, and resolves a recent open question in offline RL.
We then establish an $\Omega(H^3S\min\{C^\star, A\}/\varepsilon^2)$ sample complexity lower bound for \emph{any} policy finetuning algorithm, including those that can adaptively explore the environment. This implies that---perhaps surprisingly---the optimal policy finetuning algorithm is either offline reduction or a purely online RL algorithm that does not use $\mu$.
Finally, we design a new hybrid offline/online algorithm for policy finetuning that achieves better sample complexity than both vanilla offline reduction and purely online RL algorithms, in a relaxed setting where $\mu$ only satisfies concentrability partially up to a certain time step.
Overall, our results offer a quantitative understanding on the benefit of a good reference policy, and make a step towards bridging offline and online RL. 

\end{abstract}

\section{Introduction}

Reinforcement learning (RL)---where agents learn to play sequentially in an environment to maximize a cumulative reward function---has achieved great recent success in many artificial intelligence challenges such as video games playing~\citep{mnih2015human,vinyals2019grandmaster}, large-scale strategy games (e.g. GO)~\citep{silver2016mastering,silver2017mastering}, robotic manipulation~\citep{akkaya2019solving,lee2020learning}, behavior learning in social scenarios~\citep{baker2019emergent}, and more. In many such challenging domains, achieving human-like or superhuman performance requires training the RL agent with millions of samples (steps of acting or game playing) or more. Understanding and improving the sample efficiency of RL algorithms has been a central topic of research.



Sample-efficient RL has been studied in a rich body of theoretical work in two main settings: \emph{online RL}, in which the learner has interactive access to the environment and can execute any policy; and \emph{offline RL}, in which the learner only has access to an ``offline'' dataset collected by executing some (one or many) policies within the environment, and is not allowed to further access the environment. These two settings share some common learning goals such as the sample complexity (number of episodes of playing) for finding the optimal policy. However, existing algorithms and theories in the online and offline setting seem rather different and disconnected---In online RL, state-of-the-art sample-efficient algorithms typically explore the entire environment, e.g. by using optimism to encourage visitation to unseen states and actions~\citep{brafman2002r,kearns2002near,jaksch2010near,osband2014model, jiang2017contextual, azar2017minimax, jin2018q, dann2019policy, jin2020provably,wang2020reinforcement}. In contrast, offline RL does not allow interactive exploration, and sample-efficient policy optimization algorithms typically focus on optimizing an unbiased (or downward biased) estimator of the value function~\citep{munos2003error,szepesvari2005finite,antos2008learning,munos2008finite,chen2019information,xie2020q,liu2020provably,yin2020near,jin2020pessimism,rashidinejad2021bridging}. It is therefore of interest to ask whether these two types of algorithms and theories can be connected in any way.

Further, on the empirical end, insights and patterns from offline RL often help as well in designing online RL algorithms and improving the sample efficiency in the real world. For example, there are online RL algorithms that alternate between data collection steps using a fixed policy, and policy improvement steps by learning on the collected dataset~\citep{janner2019trust}. The replay buffer in value-based algorithms can also be seen as a local form of offline (off-policy) policy optimization and are often be used in conjunction with optimistic exploration techniques~\citep{mnih2015human,hessel2018rainbow,taiga2019benchmarking}. The prevalence of these algorithms also offers practical motivations for us to look for a more unified understanding of online and offline RL in theory. These reasonings motivate us to ask the following question:
\begin{quote}
\emph{Can we bridge sample-efficient offline and online RL from a theoretical perspective?}
\end{quote}



This paper proposes \emph{policy finetuning}, a new RL setting that investigates the benefit of a good initial policy in reinforcement learning, and encapsulates challenges of both online and offline RL. In the policy finetuning problem, the learner is given interactive access to the environment and asked to learn a near-optimal policy, but in addition has access to a \emph{reference policy} $\mu$ that is good in certain aspects. This setting offers great flexibility for the algorithm design: For example, the algorithm is allowed to either simply collect data from $\mu$ and run any offline policy optimization algorithm on the collected dataset. It is also allowed to play any other policy interactively, including those that adaptively explores the environment. The policy finetuning problem offers a common playground for both offline and online types of algorithms, and has a unified performance metric (sample complexity for finding the near-optimal policy) for comparing their performance.


We study the policy finetuning problem theoretically in finite-horizon Markov Decision Processes (MDPs) with $H$ time steps, $S$ states, and $A$ actions. We summarize our contributions as follows.
\begin{itemize}[leftmargin=1.5pc]
\item We begin by considering \emph{offline reduction} algorithms which simply collect data using the reference policy $\mu$ and run an offline policy optimization algorithm on the collected dataset. This setting equivalent to offline RL with behavior policy $\mu$, and thus our result translates to a same result for offline RL as well. 

  

We design an algorithm \pevi~that is able to find an $\eps$-optimal policy (for small $\eps$) within $\wt{O}(H^3SC^\star/\eps^2)$ episodes of play, where $C^\star$ is the \emph{single-policy concentrability} coefficient between $\mu$ and some optimal policy $\pi_\star$ (Section~\ref{section:offline}). This improves over the best existing offline result by an $H^2$ factor in the same setting and matches the lower bound (up to log factors), thereby resolving the recent open question of~\citep{rashidinejad2021bridging} on tight offline RL under single-policy concentrability.

\item Under the same assumption on $\mu$, we establish an $\Omega(H^3S\min\set{C^\star, A}/\eps^2)$ sample complexity lower bound for \emph{any} policy finetuning algorithm, including those that adaptively explores the environment (Section~\ref{section:online-lower}). This implies that the optimal policy finetuning algorithm is either offline reduction via \pevi, or a ``purely'' online RL algorithm from scratch (such as UCBVI), depending on whether $C^\star\le A$. This comes rather surprising, as it rules out possibilities of combining online exploration and knowledge of $\mu$ to further improve the sample complexity over the aforementioned two baselines.

\item Finally, we consider policy finetuning in a more challenging setting where $\mu$ only satisfies concentrability up to a certain time step. We design a ``hybrid offline/online'' algorithm \hoovi~that combines online exploration and offline data collection, and show that it achieves better sample complexity than both vanilla offline reduction and purely online algorithms in certain cases (Section~\ref{section:online-upper}). This gives a positive example on when such hybrid algorithm designs are beneficial.
\end{itemize}

\subsection{Related work}

\paragraph{Sample-efficient online RL}
There is a long line of work on establishing provably sample-efficient online RL algorithms. A major portion of these works is concerned with the tabular setting with finitely many states and actions~\citep{brafman2002r,kearns2002near,jaksch2010near, azar2017minimax, dann2017unifying, agrawal2017optimistic, jin2018q, zhang2020almost}. For episodic MDPs with inhomogeneous transition functions with $S$ states, and $A$ actions, and horizon length $H$, the optimal sample complexity for finding the $\eps$ near-optimal policy is $\wt{O}(H^3SA/\eps^2)$, achieved by various algorithms such as UCBVI of~\citet{azar2017minimax} and UCB-Advantage of~\citet{zhang2020almost}. Our paper adapts the reference-advantage decomposition technique of~\citet{zhang2020almost} to designing sharp offline algorithms. Online RL with with large state/action spaces are also studied by using function approximation in conjunction with structural assumptions on the MDP~\citep{jin2020provably, zanette2020learning, zanette2020provably,agarwal2020flambe,osband2014model, jiang2017contextual,sun2019model,wang2020reinforcement,yang2020bridging,du2021bilinear,jin2021bellman}.

\paragraph{Offline RL}
Offline/batch RL studies the case where the agent only has access to an offline dataset obtained by executing a \emph{behavior policy} in the environment. Sample-efficient learning results in offline RL typically work by assuming either sup-concentrability assumptions~\citep{munos2003error, szepesvari2005finite,antos2008learning,munos2008finite,farahmand2010error,tosatto2017boosted,chen2019information,xie2020q}) or lower bounded exploration constants~\citep{yin2020near,yin2021near} to ensure the sufficient coverage of offline data over all (relevant) states and actions. However, such strong coverage assumptions can often fail to hold in practice~\citep{fujimoto2019off}. 
More recent works address this by using either policy constraint/regularization~\citep{fujimoto2019off,liu2020provably,kumar2019stabilizing,wu2019behavior}, or the pessimism principle to optimize conservatively on the offline data~\citep{kumar2020conservative,yu2020mopo,kidambi2020morel,jin2020pessimism,yin2021near,rashidinejad2021bridging}. The policy-constraint/regularization-based approaches prevent the policy to visit states and actions that has no or low coverage from the offline data. Our proposed offline RL algorithm \pevi~(Algorithm~\ref{algorithm:pevi-adv}) is inspired by the pessimistic value iteration algorithms of~\citep{jin2020pessimism, rashidinejad2021bridging} and achieves an improved sample complexity over these work under the same single-policy concentrability assumption on the behavior policy.
%

\paragraph{Bridging online and offline RL}
\citet{kalashnikov2018scalable} observed empirically that the performance of policies trained purely from offline data can be improved considerably by a small amount of additional online fine-tuning. A recent line of work studied low switching cost RL~\citep{bai2019provably,zhang2020almost,gao2021provably,wang2021provably}---which forbits online RL algorithms from switching its policy too often---as an interpolation between the online and offline settings. The same problem is also studied empirically as deployment-efficient RL~\citep{matsushima2020deployment,su2021musbo}. While we also attempt to bridge online and offline RL, our work differs from this line in that our policy finetuning setting allows a direct comparison between ``fully offline'' and ``fully online'' algorithms, whereas the low switching cost setting prohibits fully online algorithms.






\section{Preliminaries}

\paragraph{Markov Decision Processes}
In this paper, we consider episodic Markov decision processes (MDPs) with time-inhomogeneous transitions, specified by $M = (\mc{S}, \mc{A}, H, \P, r)$, where $\mc{S}$ is the state space, $\mc{A}$ is the action space, $H$ is the horizon length, $\P = \{\P_h\}_{h = 1}^{H}$ where $\P_h(\cdot|s,a)\in\Delta_{\mc{S}}$ is the transition probabilities at step $h$, and $r = \{r_h:\mc{S}\times\mc{A} \to [0,1]\}_{h = 1}^{H}$ are the deterministic\footnote{While we assume deterministic rewards for simplicity, our results can be straightforwardly generalized to stochastic rewards, as the major difficulty is in learning the transitions rather than learning the rewards.} reward functions at time step $h \in [H]$. Without loss of generality, we assume that the initial state $s_1$ is deterministic\footnote{Any MDP with stochastic $s_1$ is equivalent to an MDP with deterministic by creating a dummy initial state $s_0$ and increasing the horizon by 1.}.

%

\paragraph{Policies, value functions, visitation distributions}
A policy $\pi=\set{\pi_h(\cdot|s)}_{h\in[H],s\in\mc{S}}$ consists of distributions $\pi_h(\cdot|s)\in\Delta_{\mc{A}}$. We use $\E_\pi[\cdot]$ to denote the expectation with respect to the random trajectory induced by $\pi$ in the MDP $M$, that is, $(s_1, a_1, r_1, s_2, a_2, r_2, \dotsc, s_H, a_H, r_H)$, where $a_h = \pi_h(s_h)$, $r_h = r_h(s_h,a_h)$, $s_{h + 1} \sim \P_h(\cdot|s_h,a_h)$. For each policy $\pi$, let $V^\pi_h:\mc S \to \R$ and $Q^\pi_h:\mc S \times \mc A \to \R$ denote its value functions and Q functions at each time step $h \in [H]$, that is,
\begin{align}
V^\pi_h(s) \coloneqq \E_\pi \bigg[\sum_{h' = h}^{H} r_{h'}(s_{h'},a_{h'}) \bigg| s_h = s \bigg], ~ Q^\pi_h(s,a) \coloneqq \E_\pi \bigg[\sum_{h' = h}^{H} r_{h'}(s_{h'},a_{h'}) \bigg| s_h = s, a_h = a \bigg].
\end{align}
The operators $\P_h$ and $\V_h$ are defined as $[\P_h V_{h + 1}](s,a) \coloneqq \E[V_{h + 1}(s')|s_h = s, a_h = a]$ and $[\V_h V_{h + 1}](s,a) \coloneqq \Var[V_{h + 1}(s')|s_h = s, a_h = a]$ for any value function $V_{h + 1}$ at time step $h + 1$. We also use $\hat{\P}_h$ and $\hat{V}_h$ to denote empirical versions of these operators building on estimated models (which will be clear in the context).

We use $\pi_\star \coloneqq \argmax_{\pi} V_1^\pi(s_1)$ to denote any optimal policy, and $V^\star_h \coloneqq V^{\pi_\star}_h$ and $Q^\star_h \coloneqq Q^{\pi_\star}_h$ to denote the value function and Q function of $\pi^\star$ at all $h \in [H]$. Throughout this paper, our learning goal is to find an near-optimal policy $\hat{\pi}$ such that $V_1^\star(s_1) - V_1^{\hat{\pi}}(s_1)\le\epsilon$.

Finally, we let $d^\pi_h$ denote the state(-action) visitation distributions of $\pi$ at time step $h \in [H]$:
\begin{align}
d^\pi_h(s) \coloneqq \P(s_h = s|\pi),~~~\text{and}~ ~~d^\pi_h(s,a) \coloneqq \P(s_h = s, a_h = a|\pi).
\end{align}

\paragraph{Miscellaneous}
We use standard $O(\cdot)$ and $\Omega(\cdot)$ notation: $A=O(B)$ is defined as $A\le CB$ for some absolute constant $C>0$ (and similarly for $\Omega$). The tilded notation $A=\wt{O}(B)$ denotes $A\le CL\cdot B$ where $L$ is a poly-logarithmic factor of problem parameters.
 
%


\subsection{Policy Finetuning}
%
We now introduce the setting of \emph{policy finetuning}. A policy finetuning problem consists of an MDP $M$ and a \emph{reference policy} $\mu$. During the learning stage, the learner can perform the following two types of moves:
\begin{enumerate}[label=(\alph*), leftmargin=1.5pc]
\item Play an episode in the MDP $M$ using any policy (i.e. learner has online interactive access to $M$).
\item Access the values of the reference policy $\mu_h(a|s)$ for all $(h,s,a)$. For example, the learner can use it to sample actions $a\sim \mu_h(\cdot|s)$ for any $h,s$ for arbitrarily many times during learning.
\end{enumerate}
The goal of the learner is to output $\eps$ near-optimal policy $\hat{\pi}$ within as few episodes of play (within the MDP) as possible. 

A unique feature about the policy finetuning setting is that it allows both \emph{online interactive plays} via any online RL algorithm (not necessarily using $\mu$), as well as \emph{offline reduction} which simply collects data by executing the reference policy $\mu$ and do anything with the collected dataset. In particular, this means that any algorithm for offline policy optimization (based on offline datasets) also gives an algorithm for policy finetuning via this offline reduction. Therefore, policy finetuning offers a common playground for both online and offline type algorithms with a unified learning goal.

\paragraph{Assumption on reference policy}
Throughout most of this paper (except for Section~\ref{section:online-upper}), we consider the following assumption on the reference policy $\mu$.
\begin{assumption}[Single-policy concentrability]
  \label{assumption:one-point-c}
  The reference policy $\mu$ satisfies that
  \begin{align*}
    \max_{h\in[H], (s,a)\in\mc{S}\times\mc{A}} \frac{d^{\pi_\star}_h(s,a)}{d^\mu_h(s,a)} \le C^\star
  \end{align*}
  (with the convention $0/0=0$) for some \emph{deterministic} optimal policy $\pi_\star$ and constant $C^\star\ge 1$.
\end{assumption}
The single-policy concentrability characterizes the distance between the visitation distributions of the reference policy $\mu$ and some optimal policy $\pi^\star$. This assumption is considered in the recent work of~\citet{rashidinejad2021bridging} on offline RL and is more relaxed than previously assumed concentrability assumptions which typically requires the supremum concentrability against all possible $\pi$'s to be bounded~\citep{chen2019information}. We consider this assumption as it both allows efficient offline RL algorithms~\citep{rashidinejad2021bridging}, and is perhaps also a sensible measure of quality for the reference policy in policy finetuning.




\section{Sharp offline learning via reference-advantage decomposition}
\label{section:offline}


%




We begin by investigating the sharpest sample complexity for policy finetuning via the offline reduction approach. This requires us to design sharp offline RL algorithms that run on the dataset $\mc{D}$ collected by executing $\mu$. We emphasize that this is both an interesting offline RL question on its own right, and also important for our later discussions on lower bounds and other algorithms for policy finetuning, as the sharpest sample complexity via offline reduction provides a solid baseline.


%

%




\paragraph{Warm-up: \vilcb}
As a warm-up, we first show that a finite-horizon variant of the \vilcb~(Value Iteration with Lower Confidence Bounds) algorithm of~\citet{rashidinejad2021bridging} achieves sample complexity $\wt{O}(H^5SC^\star/\eps^2)$ for finding an $\eps$ near-optimal policy. This result is similar to the $\wt{O}(SC^\star/(1-\gamma)^5\eps^2)$ guarantee\footnote{\citep{rashidinejad2021bridging} can achieve a faster rate in case $C^\star\le 1+\wt{O}(1/N)$. However, we focus on the case $C^\star=1+\Theta(1)$ where the guarantee of~\citet{rashidinejad2021bridging} is $\wt{O}(SC^\star/(1-\gamma)^5\eps^2)$.} for the original \vilcb~in infinite-horizon discounted MDPs~\citep[Theorem 6]{rashidinejad2021bridging}. The main ingredients of our \vilcb~algorithm is a pessimistic value iteration procedure in which we perform value iteration on the empirical model estimated from the dataset $\mc{D}$, along with a negative Hoeffding bonus term to impose pessimism. Due to space constraints, the algorithm description (Algorithm~\ref{algorithm:vilcb}) and the proof of Theorem~\ref{theorem:vilcb} are deferred to Appendix~\ref{appendix:proof-vilcb}.
\begin{theorem}[\vilcb~for finite-horizon MDPs]
\label{theorem:vilcb}
  Suppose the reference policy $\mu$ satisfies the single-policy concentrability (Assumption~\ref{assumption:one-point-c}). Then with probability at least $1-\delta$, \vilcb~(Algorithm~\ref{algorithm:vilcb}) outputs a policy $\what{\pi}$ and value estimate $\hat{V}$ such that
  \begin{enumerate}[label=(\alph*), leftmargin=2.5pc]
  \vspace{-0.0em}
  \item $\max_{h\in[H]} \sum_{s\in\mc{S}} d^{\pi_\star}_h(s) (V^\star_h(s) - \hat{V}_h(s)) \le \eps$,
  \item $V_1^\star(s_1) - V_1^{\what{\pi}}(s_1) \le \eps$,
  \end{enumerate}
  \vspace{-0.0em}
 within $n=\wt{O}\paren{H^5SC^\star/\eps^2}$ episodes.
\end{theorem}
Theorem~\ref{theorem:vilcb} serves two main purposes. First, the $\wt{O}(H^5SC^\star/\eps^2)$ sample complexity asserted in Theorem~\ref{theorem:vilcb}(b) provides a first result for offline RL (and offline reduction for policy finetuning) under single-policy concentrability in finite-horizon MDPs. Second, the value estimation bound in Theorem~\ref{theorem:vilcb}(a) shows that the estimated value function $\hat{V}_h(s)$ provided by \vilcb~is close to the optimal value $V^\star_h(s)$ \emph{at every step}~$h\in[H]$, in terms of the weighted average with $d^{\pi_\star}_h(s)$. Our next algorithm \pevi~builds on this property so that \vilcb~can be used as a ``warm-up'' learning procedure that provides a high-quality value estimate.



%



\paragraph{Sharp offline learning via reference-advantage decomposition}
We now design a new sharp algorithm \peviadv~which achieves an improved $\wt{O}(H^3SC^\star/\eps^2)$ sample complexity (for small enough $\eps$). This improves over \vilcb~by $\wt{O}(H^2)$ and is the first algorithm that matches the sample complexity lower bound. \pevi~adds two new ingredients over \vilcb~in order to achieve the $\wt{O}(H^2)$ improvement:
\begin{enumerate}[leftmargin=1.5pc]
\item We replace the Hoeffding-style bonus in \vilcb~with a Bernstein-style bonus. This shaves off one $H$ factor in the sample complexity via the total variance property (Lemma~\ref{lemma:total-variance}).
\item Both \vilcb~and our \pevi~use data splitting to make sure that the estimated value $\hat{V}_{h+1}$ and empirical transitions $\hat{\P}_h$ are estimated using different subsets of $\mc{D}$, this yields conditional independence that is required in bounding concentration terms of the form $(\hat{\P}_h-\P_h)\hat{V}_{h+1}$. However, applied naively, this data splitting induces one undesired $H$ factor in the sample complexity as we need to split $\mc{D}$ into $H$ folds and thus each $\P_h$ is estimated using only $n/H$ episodes of data.

  As a technical crux of this algorithm, we overcome this issue by adapting the \emph{reference-advantage decomposition} technique of~\citet{zhang2020almost}. This technique proposes to learn an initial reference value function $\hat{V}^{\rf}$ of good quality in a certain sense, and then performing the following type of approximate value iteration (using the right-hand side as the algorithm update):
  \begin{align*}
    \P_h\hat{V}_{h+1} \approx \hat{\P}_{h,0}\hat{V}_{h+1}^{\rf} + \hat{\P}_{h,1}\paren{\hat{V}_{h+1} - \hat{V}^{\rf}_{h+1}}.
  \end{align*}
  Above, $\hat{V}_{h+1}$, $\hat{\P}_{h,0}$, and $\hat{\P}_{h,1}$ are estimated on three disjoint subsets of the data. The advantage of this approach is that, due to this new independence structure, $\hat{\P}_{h,0}$ for different $h\in[H]$ can be estimated on the same set of trajectories without $H$-fold splitting, which shaves off the $H$ factor within this part. On the other hand, estimating $\hat{\P}_{h,1}$ still requires $H$-fold splitting, yet this would not hurt the sample complexity if the magnitude of $(\hat{V}_{h+1} - \hat{V}^{\rf}_{h+1})$ is much smaller than its naive upper bound $O(H)$---we show this can be achieved by using \vilcb~to learn $\hat{V}^{\rf}$.
\end{enumerate}

\begin{algorithm}[t]
  \setstretch{1.0}
  \caption{Pessimistic Value Iteration with Reference-Advantage Decomposition (\pevi)}
  \label{algorithm:pevi-adv}
  \begin{algorithmic}[1]
    \REQUIRE Dataset $\mc{D}=\set{(s_1^{(i)}, a_1^{(i)}, r_1^{(i)}, \dots,s_H^{(i)}, a_H^{(i)}, r_H^{(i)})}_{i=1}^n$ collected by executing $\mu$ in $M$. 
    \STATE Split the dataset $\mc{D}$ into $\mc{D}_{\rf}$, $\mc{D}_{0}$ and $\set{\mc{D}_{h,1}}_{h=1}^H$ uniformly at random:
    \begin{equation*}
      n_{\rf} \coloneqq \abs{\mc{D}_{\rf}} = n/3, ~~ n_0 \coloneqq \abs{\mc{D}_0} = n/3, ~~ n_{1,h} \coloneqq \abs{\mc{D}_{h,1}} \defeq n/(3H) ~~ (n_1 \coloneqq n/3).
    \end{equation*}
    \vspace{-1em}
    \STATE Learn a reference value function $\what{V}^{\rf} \setto \vilcb(\mc{D}_{\rf})$ via \vilcb~(Algorithm~\ref{algorithm:vilcb}). ~\label{line:vilcb}
    \STATE Let $N_{h,0}(s,a)$ and $N_{h,0}(s,a,s')$ denote the visitation count of $(s,a)$ and $(s,a,s')$ at step $h$ within dataset $\mc{D}_0$. Construct empirical model estimates:
    \begin{align*}
      \what{\P}_{h,0}(s' | s,a) \setto \frac{N_{h,0}(s,a,s')}{N_{h,0}(s,a) \vee 1},~~~\textrm{and}~~~\what{r}_{h,0}(s,a) \setto r_{h}(s, a)\indic{N_{h,0}(s,a) \ge 1}.
    \end{align*}
    Similarly define $N_{h,1}(s,a)$, $N_{h,1}(s,a,s')$, $(\what{r}_{h,1},\what{\P}_{h,1})$ for all $h\in[H]$ based on dataset $\mc{D}_{h,1}$. 
    \STATE Set
    $ b_{h,0}(s,a) \setto c \cdot \bigg(\sqrt{\frac{ [\hat{\V}_{h,0}  \hat{V}^{\rf}_{h+1}](s,a) \iota}{N_{h,0}(s,a)\vee 1}} +
      \frac{H \iota}{N_{h,0}(s,a)\vee 1}\bigg)$
    for all $(h,s,a)$,
    where $\iota\defeq \log(HSA/\delta)$. 
    \STATE Set $\hat{V}_{H+1}(s)\setto 0$ for all $s\in\mc{S}$.
    \FOR{$h=H,\dots,1$}
    \STATE Set
    $
      b_{h,1}(s,a) \setto c \cdot \bigg(\sqrt{\frac{ [\hat{\V}_{h,1} (\hat{V}_{h+1} - \hat{V}^{\rf}_{h+1})](s,a) \iota}{N_{h,1}(s,a)\vee 1}} + \frac{H \iota}{N_{h,1}(s,a)\vee 1}\bigg)$.
    \STATE Perform pessimistic value update for all $(s,a)$:
    \begin{align*}
        & \hat{Q}_h(s,a) \setto \what{r}_{h,0}(s,a) + \brac{\what{\P}_{h,0} \hat{V}^{\rf}_{h+1}}(s,a) - b_{h,0}(s,a) + \brac{\what{\P}_{h,1} (\hat{V}_{h+1} - \hat{V}^{\rf}_{h+1})}(s,a) - b_{h,1}(s,a); \\
        & \hat{V}_h(s) \setto \brac{\max_a \hat{Q}_h(s, a)} \vee 0.
    \end{align*}
    \vspace{-1em}
    \STATE Set $\what{\pi}_h(s)\setto \argmax_a \hat{Q}_h(s,a)$ for all $s\in\mc{S}$.
    \ENDFOR
    \RETURN Policy $\what{\pi}=\set{\what{\pi}_h}_{h\in[H]}$.
  \end{algorithmic}
\end{algorithm}



We instantiate this plan by carefully using \vilcb~to learn the reference value function $\hat{V}^{\rf}$, combined with tight Bernstein bonuses, to shave off another $H$ factor in the sample complexity. The full \pevi~algorithm is provided in Algorithm~\ref{algorithm:pevi-adv}. We now present its guarantee in the following theorem. The proof can be found in Appendix~\ref{sec:mainthmproof}.
\begin{theorem}[Sharp offline learning via~\pevi]
  \label{theorem:main}
  Suppose the reference policy $\mu$ satisfies the single-policy concentrability (Assumption~\ref{assumption:one-point-c}). Then with probability at least $1-\delta$, \pevi~(Algorithm~\ref{algorithm:pevi-adv}) outputs a policy $\what{\pi}$ and value estimate $\hat{V}$ such that
  \begin{enumerate}[label=(\alph*), leftmargin=2.5pc]
  \item $\max_{h\in[H]} \sum_{s\in\mc{S}} d^{\pi_\star}_h(s) (V^\star_h(s) - \hat{V}_h(s)) \le \eps$,
  \item $V_1^\star(s_1) - V_1^{\what{\pi}}(s_1) \le \eps$,
  \end{enumerate}
  within $n=\wt{O}\paren{ H^3SC^\star/\eps^2 + H^{5.5}SC^\star/\eps}$ episodes.
\end{theorem}

\paragraph{Near-optimal offline RL under single-policy concentrability}
For small enough $\eps\le H^{-2.5}$, Theorem~\ref{theorem:main} achieves $\wt{O}(H^3SC^\star/\eps^2)$ sample complexity for finding the $\eps$ near-optimal policy from the offlien dataset $\mc{D}$. This is the first cubic horizon dependence for offline RL under single-policy concentrability, which improves over recent works~\citep{jin2020pessimism,rashidinejad2021bridging} in this setting and resolves the open question of~\citep{rashidinejad2021bridging}. For $C^\star\ge 2$, our sample complexity further matches the information-theoretical lower bound $\Omega(H^3SC^\star/\eps^2)$ up to log factors\footnote{This lower bound can be adapted directly from a $\Omega(SC^\star/(1-\gamma)^3\eps^2)$ lower bound of~\citep[Theorem 7]{rashidinejad2021bridging}.}. We remark that tight hoziron dependence has also been achieved in several recent works offline RL~\citep{yin2020near,yin2021near,ren2021nearly} which are however quite different from (and do not imply) ours in both the assumptions (on the behavior policy) and the analyses.

\section{Lower bound for policy finetuning}
\label{section:online-lower}

We now switch gears to considering the policy finetuning problem with any algorithm, not necessarily restricted to the offline reduction approach.

\paragraph{Two baselines: offline reduction \& purely online RL}
A first observation is that naive offline reduction is already a strong baseline for policy finetuning, by our Theorem~\ref{theorem:main}: Our \pevi~algorithm only collects data with $\mu$ and does not do any online exploration, yet achieves a sharp $\wt{O}(H^3SC^\star/\eps^2)$ sample complexity for finding a near-optimal policy.

On the other hand, as the policy finetuning setting allows online interaction, \emph{purely online RL} is another baseline algorithm: Simply run any sample-efficient online RL algorithm (which typically uses optimism to encourage exploration) from scratch, and disregard the reference policy $\mu$. Using any sharp online RL algorithm such as UCBVI~\citep{azar2017minimax}, this approach can find an $\eps$ near-optimal policy within $\wt{O}(H^3SA/\eps^2)$ episodes of play. Note that whether this is advantageous over the offline reduction boils down to the comparison between $C^\star$ and $A$, which makes sense intuitively. For example, $C^\star\le o(A)$ means that $\mu$ is perhaps close enough to $\pi_\star$ so that collecting data from $\mu$ and run offline policy optimization is a stronger algorithm than exploring from scratch.


Given these two baselines, it is natural to ask whether there exists an algorithm that improves over both --- Can we design an algorithm that performs some amount of optimistic exploration, yet also utilizes the knowledge of $\mu$, so as to achieve a better rate than both offline reduction and purely online RL? In this section, we provide an information-theoretic lower bound showing that, perhaps surprisingly, the answer is negative: there is an $\Omega(H^3S\min\set{C^\star, A}/\eps^2)$ sample complexity lower bound for any policy finetuning algorithm, if we still assume that $\mu$ satisfies $C^\star$ single-policy concentrability.


\paragraph{Lower bound}
To formally state our lower bound, we define the class of problems
\begin{align}
  \label{equation:online-c-class}
  \mc{M}_{C^\star} \defeq \set{ (M, \mu):~\textrm{Exists deterministic $\pi_\star$ of $M$ such that}~\sup_{h,s,a}\frac{d^{\pi_\star}_h(s, a)}{d^\mu_h(s,a)} \le C^\star}.
\end{align}
We recall that a policy finetuning algorithm for problem $(M,\mu)$ is defined as any algorithm that can play in the MDP $M$ for $n$ episodes, has full knowledge of the reference policy $\mu$, and outputs a policy $\hat{\pi}$ after playing in the MDP.

With these definitions ready, we now state our lower bound for policy finetuning. The proof of Theorem~\ref{theorem:online-lower} can be found in Appendix~\ref{appendix:proof-online-lower}.

\begin{theorem}[Lower bound for policy finetuning]
  \label{theorem:online-lower}
  Suppose $S,H\ge 3$, $A\ge 2$, $C^\star\ge 2$. Then, there exists an absolute constant $c_0>0$ such that for any $\epsilon\le 1/12$ and any online finetuning algorithm that outputs a policy $\what{\pi}$, if the number of episodes
  \begin{align*}
    n \le c_0 \cdot H^3S\min\set{C^\star, A} / \eps^2,
  \end{align*}
  then there exists a problem instance $(M,\mu)\in\mc{M}_{C^\star}$ on which the algorithm suffers from $\eps$-suboptimality:
  \begin{align*}
    \E_{M}\brac{ V_{1,M}^\star - V_{1, M}^{\what{\pi}}} \ge \eps,
  \end{align*}
  where the expectation $\E_M$ is w.r.t. the randomness during the algorithm execution within MDP $M$.
\end{theorem}
\paragraph{Either offline reduction or purely online is optimal}
Theroem~\ref{theorem:online-lower} shows that any policy finetuning algorithm needs to play at least $\Omega(H^3S\min\{C^\star,A\}/\eps^2)$ episodes in order to find an $\eps$ near-optimal policy. Crucially, this implies that either a sharp offline reduction (e.g. our \pevi~algorithm) or purely online RL matches the lower bound (up to log), depending on whether $C^\star\lesssim A$. In other words, if we have the knowledge of whether $C^\star \le A$, choosing the right one of these two baseline algorithms will yield the optimal sample complexity.
Perhaps surprisingly, this rules out the possibility of designing any algorithm ``in between'' that combines online exploration and knowledge of $\mu$ to improve the sample complexity, at least in the worst-case over all problems in $\mc{M}_{C^\star}$. We argue that this ``no algorithm in between'' phenomenon may be due to the single-policy concentrability assumption being too strong such that offline reduction already achieves a rather competitive sample complexity $\wt{O}(H^3SC^\star/\eps^2)$. We investigate policy finetuning beyond the single-policy concentrability assumption in Section~\ref{section:online-upper}.



We also remark that Theorem~\ref{theorem:online-lower} generalizes both the $\Omega(H^3SA/\eps^2)$ lower bound for online RL~\citep{dann2017unifying,yin2020near,domingues2021episodic} into the policy finetuning problem, as well as the $\Omega(H^3SC^\star/\eps^2)$ lower bound for offline RL under single-policy concentrability with $C^\star\ge 2$~\citep{rashidinejad2021bridging}\footnote{The lower bound in~\citep{rashidinejad2021bridging} is $\Omega(SC^\star/\eps^2(1-\gamma)^3)$ for the infinite-horizon $\gamma$-discounted setting, which corresponds to an $\Omega(H^3SC^\star/\eps^2)$ lower bound for our finite-horizon setting.}. Further, Theorem~\ref{theorem:online-lower} directly implies an $\Omega(H^3SC^\star/\eps^2)$ lower bound for offline RL with $2\le C^\star\le O(A)$, as any algorithm for offline policy optimization is also an algorithm for policy finetuning via the offline reduction.


\paragraph{Proof intuition; Construction of hard instance}
The proof of Theorem~\ref{theorem:online-lower} constructs a family of hard MDPs that requires solving $HS$ ``independent'' bandit problems with $A$ arms, similar as in existing $\Omega(H^3SA/\eps^2)$ lower bounds for online RL~\citep{dann2017unifying,yin2020near}. However, our key modification is that we let the optimal arms to be always within the first $K\defeq \min\set{C^\star,A}$ actions instead of all $A$ actions, and we define our reference policy $\mu$ to play uniformly within $[K]$. This $\mu$ has the following properties:
\begin{itemize}[leftmargin=1.5pc]
\item $\mu$ satisfies $C^\star$ single-policy concentrability for any MDP in this family (Lemma~\ref{lemma:ref-concentrability-online}).
\item $\mu$ provides the knowledge that the optimal actions are within $[K]$, but \emph{no other knowledge} about the optimal actions. 
\end{itemize}
Therefore, with $\mu$ at hand, any policy finetuning algorithm can ``gain the knowledge'' that the optimal actions are within $[K]$, but still needs to try all $K$ actions in order to solve each bandit problem---rigorizing this information-theoretically gives the $\Omega(H^3SK/\eps^2)=\Omega(H^3S\min\set{C^\star,A}/\eps^2)$ lower bound.

\section{Hybrid offline/online algorithm for policy finetuning}
\label{section:online-upper}

\begin{algorithm}[t]
  \setstretch{1.0}
  \caption{Hybrid Offline/Online Value Iteration (\hoovi)}
  \label{algorithm:hybrid}
  \begin{algorithmic}[1]
    \REQUIRE MDP $M$, reference policy $\mu$.
    \STATE {\color{blue} \# Stage 1: Learn step $h_\star+1:H$ via optimistic online exploration}
    \FOR{Episode $k=1,\dots,n_{\rm UCB}=n/2$}
    \STATE Receive initial state $s_1$ and play with policy $\mu$ up to step $h_\star$. Arrive at state $s_{h_\star+1}$.
    \STATE Play step $h_\star+1$ to $H$ using the \ucbviuplow~algorithm (Algorithm~\ref{algorithm:ucbviuplow}).
    \ENDFOR
    \STATE Denote the final output of \ucbviuplow~as
    \begin{align*}
      (\overline{V}_{h_\star+1}, \underline{V}_{h_\star+1}, \what{\pi}^{\rm UCB}_{(h_\star+1):H})\setto \ucbviuplow(n_{\rm UCB}).
    \end{align*}

    \STATE {\color{blue} \# Stage 2: Learn step $1:h_\star$ via executing $\mu$ + pessimistic offline policy optimization}
    \STATE Collect $\mc{D}\setto$ $\{n-n_{\rm UCB}$ episodes of data using policy $\mu$ up to step $h_\star\}$.
    \STATE Learn policy $\what{\pi}^{\rm PEVI}_{1:h_\star}$ via the \truncpevi (Algorithm~\ref{algorithm:trunc-pevi}):
    \begin{align*}
      \hat{\pi}^{\rm PEVI}_{1:h_\star} \setto
      \text{\truncpevi}(\mc{D}, h_\star, \underline{V}_{h_\star+1}).
    \end{align*}
    \vspace{-1em}
    \RETURN Policy $\what{\pi}=(\what{\pi}^{\rm PEVI}_{1:h_\star}, \what{\pi}^{\rm UCB}_{(h_\star+1):H})$.
  \end{algorithmic}
\end{algorithm}

Towards circumventing the lower bound in Theorem~\ref{theorem:online-lower}, in this section, we study policy finetuning under more relaxed assumptions on the reference policy $\mu$. A weaker $\mu$ will induce a higher sample complexity for naive offline reduction approaches, and thus yields opportunities for designing new algorithms that can potentially better utilize $\mu$.


More concretely, we consider the following relaxation: We assume $\mu$ satisfies \emph{partial concentrability} only up to a certain time-step $h_\star\le H$, and may not have any bounded concentrability at steps $h>h_\star$. We formalize this in the following
\begin{assumption}[$h_\star$-partial concentrability]
  \label{assumption:partial-c}
  The reference policy $\mu$ satisfies the single-policy concentrability with respect to $\pi_\star$ \emph{up to step $h_\star$ only}:
  \begin{align*}
    \max_{h\le h_\star} \max_{s,a\in\mc{S}\times\mc{A}} \frac{d^{\pi_\star}_h(s,a)}{d^\mu_h(s,a)} \le C^{\rm partial}
  \end{align*}
  (with the convention $0/0=0$), where $\pi_\star$ is some deterministic optimal policy of the MDP, and constant $C^{\rm partial}\ge 1$.
\end{assumption}

\paragraph{Algorithm description}
We design a hybrid offline/online algorithm \hoovi~(presented in Algorithm~\ref{algorithm:hybrid}) for policy finetuning under the partial concentrability assumption. At a high-level, the algorithm consists of two main stages:
\begin{itemize}[leftmargin=1.5pc]
\item In the first stage, it runs an online algorithm \ucbviuplow~which uses optimistic exploration to find a near-optimal policy $\hat{\pi}^{\rm UCB}$ and an accurate value estimate for steps $(h_\star+1):H$.
\item In the second stage, we run a \truncpevi~algorithm, which collects data from $\mu$ and runs offline policy optimization to find a near-optimal policy $\hat{\pi}^{\rm PEVI}$ for steps $1:h_\star$, \emph{building on the lower value estimate $\low{V}_{h_\star+1}$ from the first stage}.
\end{itemize}
This strategy makes sense intuitively as the reference policy $\mu$ does not have guarantees for steps $h_\star+1:H$ and thus the algorithm is required to perform optimistic exploration first to get a good policy.
However, additional technical cares are needed in order to make the above algorithm provably sample-efficient. The analysis of the second stage requires the online algorithm in the first stage to not only perform fast exploration (e.g. by using upper confidence bounds), but also output a \emph{lower value estimate} for step $h_\star+1$, and in addition output a final output policy that achieves at least the value of the lower value estimate \emph{at every state $s\in\mc{S}$}. Such lower bounds are not directly available in standard online RL algorithms such as UCBVI~\citep{azar2017minimax}. 

We resolve this by designing the \ucbviuplow~algorithm (detailed description in Algorithm~\ref{algorithm:ucbviuplow}), which is a modification of the Nash-VI Algorithm of~\citet{liu2020sharp} (for two-player Markov games) into the single-player case. This algorithm is particularly suitable for our purpose since it maintains both upper bounds of $V^\star$ and lower bounds for the value function of the deployed policies. Our \ucbviuplow~further integrates the certified policy technique of~\citet{bai2020near} to make sure that its output policy achieves value greater or equal than the lower bound at every state (similar guarantees can also be obtained by the policy certificate technique of~\citet{dann2019policy}).





We now state our main theoretical guarantee for the \hoovi~algorithm. The proof can be found in Appendix~\ref{appendix:proof-online-upper}.
\begin{theorem}[Hybrid online / offline learning for policy finetuning]
\label{theorem:online-upper}
Suppose the reference policy $\mu$ satisfies the partial concentrability (Assumption~\ref{assumption:partial-c}) up to some step $h_\star\le H$. Then for small enough $\eps\le \min\set{h_\star^{-2.5}, C^{\rm partial}/S}$, \hoovi~(Algorithm~\ref{algorithm:hybrid}) outputs a policy $\what{\pi}$ such that $V_1^\star(s_1) - V_1^{\hat{\pi}}(s_1) \le \eps$ with probability at least $1-\delta$, within
\begin{align*}
    n = \wt{O}\paren{ \frac{H^2h_\star SC^{\rm partial} + (H-h_\star)^3SA(C^{\rm partial})^2}{\eps^2} }
\end{align*}
episodes of play.
\end{theorem}

\paragraph{Comparison against offline reduction and purely online algorithms}
The sample complexity in Theorem~\ref{theorem:online-upper} compares favorably against both naive offline reduction as well as purely online algorithms in certain situations. First, naive offline reduction with $\mu$ does not have any guarantee since $\mu$ is not assumed to have a finite single-policy concentrability at $h\ge h_\star+1$. We can modify $\mu$ into $\mu'$ that plays uniformly within $\mc{A}$ at steps $h\ge h_\star+1$; the single-policy concentrability coefficient of $\mu'$ is guaranteed to be finite but scales exponentially as $O(A^{H-h_\star})$ in the worst case, leading to a sample complexity much worse than ours (which is polynomial in $H,S,A$).

On the other hand, a sharp online algorithm can still achieve $\wt{O}(H^3SA/\eps^2)$ in this setting (by optimistic exploration from scratch). Our Theorem~\ref{theorem:online-upper} is in general incomparable with this, but can be better in cases when both $C^{\rm partial}$ and $H-h_\star$ are small, e.g., if $C^{\rm partial}=o(A)$ and $(H-h_\star)/H=o((C^{\rm partial})^{-2/3})$. This makes sense intuitively as our hybrid offline/online algorithm benefits the most if the length requiring exploration ($H-h_\star$) is small, and the partial concentrability $C^{\rm partial}$ is small so that $\mu$ still has a high-quality for the first $h_\star$ steps. To best of our knowledge, this is first result that characterizes when the sample complexity of such hybrid algorithms can be beneficial over purely online or offline algorithms.




\section{Conclusion \& discussions}
This paper studies policy finetuning, a new reinforcement learning setting that allows us to compare and connect sample-efficient online and offline reinforcement learning. We establish sharp upper and lower bounds for policy finetuning under various assumptions on the reference policy. Our bounds show that the optimal policy finetuning algorithm is either offline reduction or a purely online algorithm in the specific setting where the reference policy satisfies single-policy concentrability, and we also show that a hybrid online/offline algorithm can be advantageous over both in more relaxed settings. Many directions could be of interest for future research, such as alternative assumptions on the reference policy, or policy finetuning with function approximation.

Also, while our contributions are mainly theoretical, implementing or extending our policy finetuning algorithms on real-world RL tasks would be a compelling future direction. When the environment is a tabular MDP, our Algorithm~\ref{algorithm:pevi-adv} (offline reduction) and Algorithm~\ref{algorithm:hybrid} (hybrid offline / online RL) are readily implementable. When there is large state/action space and potentially function approximation, we believe our algorithm can be adapted, for example, by replacing all the optimistic/pessimistic value iteration steps by DQN-type algorithms~\citep{mnih2015human} with positive/negative bonus functions~\citep{taiga2020on}. Experimental evaluation of such algorithms would be a good direction for future work.

\section*{Acknowledgment}
The authors would like to thank Ming Yin, Chi Jin and David Forsyth for the many insightful discussions. NJ acknowledges funding support from the ARL Cooperative Agreement W911NF-17-2-0196, NSF IIS-2112471, and Adobe Data Science Research Award. HW, CX, YB are funded through employment with Salesforce.

\bibliographystyle{abbrvnat}
\bibliography{bib}



\makeatletter
\def\renewtheorem#1{%
  \expandafter\let\csname#1\endcsname\relax
  \expandafter\let\csname c@#1\endcsname\relax
  \gdef\renewtheorem@envname{#1}
  \renewtheorem@secpar
}
\def\renewtheorem@secpar{\@ifnextchar[{\renewtheorem@numberedlike}{\renewtheorem@nonumberedlike}}
\def\renewtheorem@numberedlike[#1]#2{\newtheorem{\renewtheorem@envname}[#1]{#2}}
\def\renewtheorem@nonumberedlike#1{  
\def\renewtheorem@caption{#1}
\edef\renewtheorem@nowithin{\noexpand\newtheorem{\renewtheorem@envname}{\renewtheorem@caption}}
\renewtheorem@thirdpar
}
\def\renewtheorem@thirdpar{\@ifnextchar[{\renewtheorem@within}{\renewtheorem@nowithin}}
\def\renewtheorem@within[#1]{\renewtheorem@nowithin[#1]}
\makeatother

\renewtheorem{theorem}{Theorem}[section]
\renewtheorem{lemma}{Lemma}[section]
\renewtheorem{remark}{Remark}
\renewtheorem{corollary}{Corollary}[section]
\renewtheorem{observation}{Observation}[section]
\renewtheorem{proposition}{Proposition}[section]
\renewtheorem{definition}{Definition}[section]
\renewtheorem{claim}{Claim}[section]
\renewtheorem{fact}{Fact}[section]
\renewtheorem{assumption}{Assumption}[section]
\renewcommand{\theassumption}{\Alph{assumption}}
\renewtheorem{conjecture}{Conjecture}[section]

\appendix
\section{Technical tools}
\label{appendix:tools}

\begin{lemma}[Binomial concentration]
  \label{lemma:binomial-concentration}
  Suppose $N\sim\Bin(n, p)$ where $n\ge 1$ and $p\in[0,1]$. Then with probability at least $1-\delta$, we have
  \begin{align*}
    \frac{p}{N\vee 1} \le \frac{8\log(1/\delta)}{n}.
  \end{align*}
\end{lemma}
\begin{proof}
  If $p\le 8\log(1/\delta)/n$, the result clearly holds regardless of the value of $N$. Suppose $p > 8\log(1/\delta)/n$, then by the multiplicative Chernoff bound, we have
  \begin{align*}
    \P\paren{N < \frac{1}{2}np} \le \exp\paren{-\frac{(1/2)^2np}{2}} = \exp(-np/8) \le \delta.
  \end{align*}
  Therefore, with probability at least $1-\delta$, we have $N\ge np/2$ and thus
  \begin{align*}
    \frac{p}{N\vee 1} \le \frac{p}{N} \le \frac{2}{n} \le \frac{8\log(1/\delta)}{n}.
  \end{align*}
  This shows the desired result.
\end{proof}
\section{Proof of Theorem~\ref{theorem:vilcb}}
\label{appendix:proof-vilcb}

Throughout the proofs we let $C>0$ denote an absolute constant that can vary from line to line. This $C$ is not to be confused with the single-policy concentrability coefficient $C^\star$ in Assumption~\ref{assumption:one-point-c}.



\subsection{Algorithm}
We first present the \vilcb~algorithm in Algorithm~\ref{algorithm:vilcb}. This algorithm is an analogue of the original \vilcb~algorithm of~\citep{rashidinejad2021bridging} for finite-horizon MDPs (instead of infinite-horizon discounted MDPs). Within the algorithm, the constant $c$ in Line~\ref{line:vilcb-c} is chosen to be the same $c>0$ as in Lemma~\ref{lem:concentration_hoeffding}.

\begin{algorithm}[t]
  \caption{Value Iteration with Lower Confidence Bounds (VI-LCB) for episodic MDPs}
  \label{algorithm:vilcb}
  \begin{algorithmic}[1]
    \REQUIRE Offline dataset $\mc{D}=\set{(s_1^{(i)}, a_1^{(i)}, r_1^{(i)}, \dots,s_H^{(i)}, a_H^{(i)}, r_H^{(i)})}_{i=1}^n$. 
    \STATE Randomly split the dataset $\mc{D}$ into $\set{\mc{D}_{h}}_{h=1}^H$ with $\abs{\mc{D}_h}=n/H$.
    \STATE Let $N_{h}(s,a)$ and $N_{h}(s,a,s')$ denote the visitation count of $(s,a)$ and $(s,a,s')$ at step $h$ within dataset $\mc{D}_h$. Construct empirical model estimates:
    \begin{align*}
      & \what{r}_{h}(s,a) \setto r_h(s,a) \indic{N_h(s,a) \ge 1}, \\
      & \what{\P}_{h}(s' | s,a) \setto \frac{N_{h}(s,a,s')}{N_{h}(s,a)\vee 1}.
    \end{align*}
    \STATE Set $\hat{V}_{H+1}(s)\setto 0$ for all $s\in\mc{S}$.
    \FOR{$h=H,\dots,1$}
    \STATE Set
    $b_{h}(s,a) \setto c \cdot \sqrt{\frac{H^2\iota}{N_{h}(s,a)\vee
        1}}$ (where $\iota\defeq \log(HSA/\delta)$). \label{line:vilcb-c}
    \STATE Perform value update for all $(s,a)$:
    \begin{align*}
        & \hat{Q}_h(s,a) \setto \what{r}_{h}(s,a) + \brac{\what{\P}_{h} \hat{V}_{h+1}}(s,a) - b_{h}(s,a); \\
        & \hat{V}_h(s) \setto \brac{\max_a \hat{Q}_h(s, a)} \vee 0.
    \end{align*}
    \STATE Set $\what{\pi}_h(s)\setto \argmax_a \hat{Q}_h(s,a)$ for all $s\in\mc{S}$.
    \ENDFOR
    \RETURN Value estimate $\hat{V}=\set{\hat{V}_h}_{h\in[H]}$, policy $\what{\pi}=\set{\what{\pi}_h}_{h\in[H]}$.
  \end{algorithmic}
\end{algorithm}

\subsection{Some lemmas}

\begin{lemma}[Concentration]
\label{lem:concentration_hoeffding}
    Under the setting of Theorem~\ref{theorem:vilcb}, there exists an absolute constant $c>0$ such that the concentration event $\mc{E}$ holds with probability at least $1-\delta$, where
    \begin{align}
        \mc{E} \defeq \Bigg\{
        & \abs{\brac{\hat{r}_h - r_h}(s, a) + \brac{(\hat{\P}_h - \P_h) \hat{V}_{h+1}}(s, a)} \le c\cdot \sqrt{\frac{H^2\iota}{N_h(s, a)\vee 1}} = b_h(s,a),~~~\textrm{and} \\
        & \frac{1}{N_h(s, a)\vee 1} \le c\cdot \frac{H\iota}{nd^{\mu}_h(s, a)}~~~\textrm{for all}~(h,s,a)\in[H]\times\mc{S}\times\mc{A}
        \Bigg\},
    \end{align}
    and $\iota\defeq \log(HSA/\delta)$.
\end{lemma}
\begin{proof}[Proof of Lemma~\ref{lem:concentration_hoeffding}]
Fix any $(h, s, a)$. The first claim holds trivially if $N_h(s, a)=0$, as $\hat{r}_h(s,a)=\hat{\P}_h(\cdot|s, a)=0$ by definition, and thus the left-hand side is upper bounded by $H+1\le 2H\le 2\sqrt{H^2\iota}$. If $N_h(s, a)\ge 1$, conditioned on $N_h(s,a)$, we have
\begin{align*}
    & \hat{r}_h(s, a) = r_h(s, a), \\
    & \abs{\brac{(\hat{\P}_h - \P_h)\hat{V}_{h+1}}(s, a)} \le c\cdot \sqrt{\frac{H^2\log(HSA/\delta)}{N_h(s, a)}},
\end{align*}
where the last inequality is obtained by the Azuma-Hoeffding inequality with probability at least $1-\delta/(2HSA)$, using the fact that the data used in obtaining $\hat{\P}_h$ is independent of the data used in obtaining $\hat{V}_{h+1}$ (due to the data splitting in Algorithm~\ref{algorithm:vilcb}). Therefore
\begin{align*}
    \abs{\brac{\hat{r}_h - r_h}(s, a) + \brac{(\hat{\P}_h - \P_h) \hat{V}_{h+1}}(s, a)} \le c\cdot \sqrt{\frac{H^2\iota}{N_h(s, a)\vee 1}}.
\end{align*}
Further taking the union bound yields the first claim over all $(h,s,a)$ with probability at least $1-\delta/2$.

For the second claim, notice that $N_h(s, a)\sim \Bin(n/H, d^\mu_h(s, a))$. Applying Lemma~\ref{lemma:binomial-concentration} yields that
\begin{align*}
    \frac{1}{N_h(s, a)\vee 1} \le \frac{8\log(2HSA/\delta)}{n/H \cdot d^\mu_h(s, a)} \le c\cdot \frac{H\iota}{nd^\mu_h(s,a)}
\end{align*}
with probability at least $1-\delta/(2HSA)$. Taking the union bound yields the second claim over all $(h,s,a)$ with probability at least $1-\delta/2$.
\end{proof}

\begin{lemma}[Monotonicity for~\vilcb]
\label{lem:algo3monot}
Let $\hat\pi$ be the output policy of Algorithm~\ref{algorithm:vilcb}. Then, on the event $\mc{E}$ defined in Lemma~\ref{lem:concentration_hoeffding}, we have
\begin{align*}
  \hat V_h(s) \leq V^{\what{\pi}}_h(s) \leq V^\star_h(s)
\end{align*}
for any $s \in \mc{S}$ and $h \in [H]$.
\end{lemma}
\begin{proof}[Proof of Lemma~\ref{lem:algo3monot}]

We first prove $\hat V_h(s) \leq V^{\what{\pi}}_h(s)$ for any $s \in \mc{S}$ and $h \in [H]$ by induction. For $h = H$, and any $s\in \mc{S},a\in \mc{A}$,
\begin{align}
&~ Q^{\what{\pi}}_H(s,a) - \hat Q_H(s,a)
\\
\label{eq:hoeff_H}
= &~ r_H(s,a) - \what{r}_H(s,a) + b_H(s,a) \geq 0,
\end{align}
where the last inequality is by Lemma~\ref{lem:concentration_hoeffding}. Then,
\begin{align}
&~ V^{\what{\pi}}_H(s) - \hat V_H(s)
\\
= &~ Q^{\what{\pi}}_H(s,\what{\pi}_H(s)) - \brac{\max_a \hat{Q}_H(s, a)} \vee 0
\\
= &~ Q^{\what{\pi}}_H(s,a_{s,H}) - \brac{\hat{Q}_H(s, a_{s,H})} \vee 0
\tag{$a_{s,H} \coloneqq \what{\pi}_H(s)$}
\\
= &~ \brac{Q^{\what{\pi}}_H(s,a_{s,H}) - \hat{Q}_H(s, a_{s,H})} \wedge Q^{\what{\pi}}_H(s,a_{s,H}) \geq 0,
\end{align}
where the last inequality follows from Lemma~\ref{lem:concentration_hoeffding}, and $Q^{\pi}_H(s,a) \in [0,1]$ for any $\pi$, $s \in \mc{S}$, and $a \in \mc{A}$.

We now show that, if $\hat V_{h + 1}(s) \leq V^{\what{\pi}}_{h + 1}(s)$ holds for any $s$, we also have $\hat V_{h}(s) \leq V^{\what{\pi}}_{h}(s)$ for any $s$. Recall that $\hat{V}_h(s) = \max_a \hat{Q}_h(s, a) \vee 0 = \hat{Q}_h(s, \hat{\pi}_h(s))\vee 0$. The claim clearly holds in the trivial case of $\hat{V}_h(s) = 0$. Otherwise, we have
\begin{align}
&~ V^{\what{\pi}}_{h}(s) - \hat V_{h}(s)
\\
= &~ Q^{\what{\pi}}_{h}(s,a_{s,h}) - \hat Q_{h}(s,a_{s,h})
\tag{$a_{s,h} \coloneqq \what{\pi}_h(s)$}
\\
\label{eq:hoeff_h_temp}
= &~ r_h(s,a_{s,h}) + \brac{\P_{h} V_{h+1}^{\what{\pi}}}(s,a_{s,h}) - \what{r}_{h}(s,a_{s,h}) - \brac{\what{\P}_{h} \hat V_{h+1}}(s,a_{s,h}) + b_{h}(s,a_{s,h})
\\
= &~ r_h(s,a_{s,h}) - \what{r}_{h}(s,a_{s,h}) + \brac{\P_{h} \left(V_{h+1}^{\what{\pi}} - \hat V_{h+1}\right) }(s,a_{s,h}) + \brac{\left(\P_{h} - \what{\P}_{h} \right) \hat V_{h+1}}(s,a_{s,h}) + b_{h}(s,a_{s,h})
\\
\label{eq:hoeff_h}
\geq &~ r_h(s,a_{s,h}) - \what{r}_{h}(s,a_{s,h}) + \brac{\left(\P_{h} - \what{\P}_{h} \right) \hat V_{h+1}}(s,a_{s,h}) + b_{h}(s,a_{s,h})
\\
\geq &~ 0,
\end{align}
where the first inequality follows from $V^{\what{\pi}}_{h + 1}(s) \geq \hat V_{h + 1}(s)$ for any $s$, and the last inequality is by Lemma~\ref{lem:concentration_hoeffding}.
This completes the proof of $V_h(s) \leq V^{\what{\pi}}_h(s), ~ \forall s\in\mc{S},h\in[H]$.

The argument $V^{\what{\pi}}_h(s) \leq V^\star_h(s)$ holds by definition of $V^\star_h(s) = V^{\pi_\star}_h(s) \geq V^{\pi}_h(s)$, for any $\pi$, $s\in\mc{S}$, and $h \in [H]$. Thus, we complete the proof.
\end{proof}

\begin{lemma}[Performance decomposition for \vilcb]
\label{lem:Vstarbd}
On the event $\mc{E}$ defined in Lemma~\ref{lem:concentration_hoeffding}, we have
\begin{align}
\sum_{s\in\mc{S}} d^{\pi_\star}_h(s) (V^\star_h(s) - \hat{V}_h(s)) \leq 2 \sum_{h' = h}^{H} \sum_{(s,a) \in \mc{S} \times \mc{A}} d_{h'}^{\pi_\star}(s,a) b_{h'}(s,a),
\end{align}
for any $h \in [H]$.
\end{lemma}
\begin{proof}[Proof of Lemma~\ref{lem:Vstarbd}]
  Throughout this proof we let $\hat{Q}_h(s, \pi_\star)\defeq \hat{Q}_h(s, \pi_{\star, h}(s))$ for shorthand. We have
\begin{align}
  &~ \sum_{s\in\mc{S}} d^{\pi_\star}_h(s) (V^\star_h(s) - \hat{V}_h(s)) \\
  \leq &~ \sum_{s\in\mc{S}} d^{\pi_\star}_h(s) (V^\star_h(s) - \max_a \hat{Q}_h(s, a))
    \tag{by definition of $\hat{V}_h(s) =\max_a \hat{Q}_h(s, a)\vee 0$}
\\
\leq &~ \sum_{s\in\mc{S}} d^{\pi_\star}_h(s) (V^\star_h(s) - \hat{Q}_h(s, \pi_\star))
\\
= &~ \sum_{h' = h}^{H} \sum_{(s,a) \in \mc{S} \times \mc{A}} d_{h'}^{\pi_\star}(s,a) r_{h'}(s,a) - \sum_{s\in\mc{S}} d_h^{\pi_\star}(s) \hat Q_h(s_1,\pi_\star)
\\
= &~ \sum_{h' = h}^{H} \sum_{(s,a) \in \mc{S} \times \mc{A}} d_{h'}^{\pi_\star}(s,a) r_{h'}(s,a) - \sum_{h' = h}^{H}\sum_{s,a,s',a'}\left( d_{h'}^{\pi_\star}(s,a) \hat Q_{h'}(s,a) - d_{h' + 1}^{\pi_\star}(s',a') \hat Q_{h' + 1}(s',a') \right)
\\
\leq &~ \sum_{h' = h}^{H} \sum_{(s,a) \in \mc{S} \times \mc{A}} d_{h'}^{\pi_\star}(s,a) r_{h'}(s,a) - \sum_{h' = h}^{H}\sum_{s,a,s'}\left( d_{h'}^{\pi_\star}(s,a) \hat Q_{h'}(s,a) - d_{h' + 1}^{\pi_\star}(s') \hat V_{h' + 1}(s') \right)
\tag{by definition of $\hat{V}_{h'}$}
\\
= &~ \sum_{h' = h}^{H} \sum_{(s,a) \in \mc{S} \times \mc{A}} d_{h'}^{\pi_\star}(s,a) r_h(s,a) - \sum_{h' = h}^{H}\sum_{(s,a) \in \mc{S} \times \mc{A}}d_{h'}^{\pi_\star}(s,a) \left( \hat Q_{h'}(s,a) - \brac{\P_{h'} \hat V_{h' + 1}}(s,a) \right)
\\
= &~ \sum_{h' = h}^{H} \sum_{(s,a) \in \mc{S} \times \mc{A}} d_{h'}^{\pi_\star}(s,a) \left( r_{h'}(s,a) + \brac{\P_{h'} \hat V_{h' + 1}}(s,a) - \hat Q_{h'}(s,a) \right)
\\
= &~ \sum_{h' = h}^{H} \sum_{(s,a) \in \mc{S} \times \mc{A}} d_{h'}^{\pi_\star}(s,a) \left( r_{h'}(s,a) + \brac{\P_{h'} \hat V_{h' + 1}}(s,a) - \what{r}_{h'}(s,a) - \brac{\hat\P_{h'} \hat V_{h' + 1}}(s,a) + b_{h'}(s,a) \right)
\\
\leq &~ 2 \sum_{h' = h}^{H} \sum_{(s,a) \in \mc{S} \times \mc{A}} d_{h'}^{\pi_\star}(s,a) b_{h'}(s,a).
\tag{by the concentration event $\mc{E}$} 
\end{align}
This completes the proof.
\end{proof}

\subsection{Proof of  main theorem}
\label{appendix:proof-vilcb-main}
We are now ready to prove the Theroem~\ref{theorem:vilcb}. We first prove part (a). By Lemma~\ref{lem:algo3monot}, we have
\begin{align}
&~ \max_{h\in[H]} \sum_{s\in\mc{S}} d^{\pi_\star}_h(s) (V^\star_h(s) - \hat{V}_h(s))
\\
\leq &~ 2 \sum_{h = 1}^{H} \sum_{(s,a) \in \mc{S} \times \mc{A}} d_h^{\pi_\star}(s,a) b_h(s,a)
\\
= &~ c \sum_{h = 1}^{H} \sum_{(s,a) \in \mc{S} \times \mc{A}} d_h^{\pi_\star}(s,a) \cdot \sqrt{\frac{H^2 \iota}{N_{h}(s,a)\vee 1}}
\\
= &~ c \sqrt{H^2 \iota} \sum_{h = 1}^{H} \sum_{(s,a) \in \mc{S} \times \mc{A}} d_h^{\pi_\star}(s,a) \cdot \sqrt{\frac{1}{N_{h}(s,a)\vee 1}}
\\
\leq &~ c \sqrt{H^2 \iota} \sum_{h = 1}^{H} \sum_{(s,a) \in \mc{S} \times \mc{A}} d_h^{\pi_\star}(s,a) \cdot \sqrt{\frac{H \iota}{n d_h^\mu(s,a)}}
\tag{by the concentration event $\mc{E}$}
\\
\leq &~ c \sqrt{H^3 \iota^2} \sum_{h = 1}^{H} \sum_{(s,a) \in \mc{S} \times \mc{A}} \sqrt{d_h^{\pi_\star}(s,a)} \cdot \sqrt{\frac{d_h^{\pi_\star}(s,a)}{n d_h^\mu(s,a)}}
\\
  \leq &~ c \sqrt{\frac{H^3 C^\star \iota^2}{n}} \sum_{h = 1}^{H} \sum_{(s,a) \in \mc{S} \times \mc{A}} \sqrt{d_h^{\pi_\star}(s,a)}
  \tag{Assumption~\ref{assumption:one-point-c}}
\\
  = &~ c \sqrt{\frac{H^3 C^\star \iota^2}{n}} \sum_{h = 1}^{H} \sum_{(s,a) \in \mc{S} \times \mc{A}} \sqrt{\1\{a = \pi_\star(s)\} \cdot d_h^{\pi_\star}(s,a)}
      \tag{$\pi_\star$ is deterministic}
\\
\leq &~ c \sqrt{\frac{H^3 C^\star \iota^2}{n}}  \sqrt{\sum_{h = 1}^{H} \sum_{(s,a) \in \mc{S} \times \mc{A}} \1\{a = \pi_\star(s)\}} \cdot \sqrt{\sum_{h = 1}^{H} \sum_{(s,a) \in \mc{S} \times \mc{A}} d_h^\star(s,a)}
\\
\tag{by Cauchy–Schwarz inequality}
  \\
  \leq &~ c \sqrt{\frac{H^3 C^\star \iota^2}{n}}  \sqrt{HS} \cdot \sqrt{H}
\\
= &~ c \sqrt{\frac{H^5 S C^\star \iota^2}{n}}.
\end{align}
Therefore, as long as $n\ge O(H^5SC^\star\iota^2/\eps^2)$, we have $\max_{h\in[H]} \sum_{s\in\mc{S}} d^{\pi_\star}_h(s) (V^\star_h(s) - \hat{V}_h(s))\le \eps$. This shows part (a). Further, this bound at $h=1$ says
\begin{align*}
  V_1^\star(s_1) - \hat{V}_1(s_1) \le \eps
\end{align*}
(as we assumed deterministic $s_1$). Part (b) follows directly from this and the fact that $V_1^{\hat{\pi}}(s_1) \ge \hat{V}_1(s_1)$ which was shown in Lemma~\ref{lem:algo3monot}.
\qed

\section{Proof of Theorem~\ref{theorem:main}}
\label{sec:mainthmproof}

\subsection{Some Lemmas}

\begin{lemma}[Concentration]
\label{lem:concentration_advref}
    Under the setting of Theorem~\ref{theorem:main}, there exists an absolute constant $c>0$ such that the concentration event $\mc{E}$ holds with probability at least $1-\delta$, where
    \begin{align}
        \mc{E} \defeq \Bigg\{
        \text{(i):} &~ \abs{\brac{\hat{r}_{h,0} - r_h}(s, a) + \brac{(\hat{\P}_{h,0} - \P_h) \hat{V}_{h+1}^\rf}(s, a)}
        \\
        &~ \le c\cdot \left( \sqrt{\frac{[\hat{\V}_{h,0} (\hat{V}^{\rf}_{h+1})](s,a) \iota}{N_{h,0}(s, a)\vee 1}} + \frac{H \iota}{N_{h,0}(s, a)\vee 1} \right) = b_{h,0}(s,a),
        \\
        \text{(ii):} &~ \abs{\brac{(\hat{\P}_{h,1} - \P_h) (\hat{V}_{h+1} - \hat{V}_{h+1}^\rf)}(s, a)}
        \\
        &~ \le c\cdot \left( \sqrt{\frac{[\hat{\V}_{h,1} (\hat{V}_{h+1} - \hat{V}_{h+1}^\rf)](s,a) \iota}{N_{h,1}(s, a)\vee 1}} + \frac{H \iota}{N_{h,1}(s, a)\vee 1} \right) = b_{h,1}(s,a),
      \\
      \text{(iii):} &~ \frac{1}{N_{h,0}(s, a)\vee 1} \le c\cdot \frac{\iota}{nd^{\mu}_h(s, a)} ~~~\textrm{and} 
      \\
      \text{(iv):} &~ \frac{1}{N_{h,1}(s, a)\vee 1} \le c\cdot \frac{H\iota}{nd^{\mu}_h(s, a)} ~~~\textrm{for all}~(h,s,a)\in[H]\times\mc{S}\times\mc{A}
        \Bigg\},
    \end{align}
    and $\iota\defeq \log(HSA/\delta)$.
\end{lemma}
\begin{proof}[Proof of Lemma~\ref{lem:concentration_advref}]
Claim (i) holds trivially if $N_{h,0}(s, a)=0$, as $\hat{r}_{h,0}(s,a)=\hat{\P}_{h,0}(\cdot|s, a)=0$ by definition, and thus the left-hand side is upper bounded by $H+1\le 2H\le H\iota$. If $N_{h,0}(s, a)\ge 1$, conditioned on $N_{h,0}(s,a)$, we have
\begin{align*}
    & \hat{r}_{h,0}(s,a) = r_h(s, a), \\
    & \abs{\brac{(\hat{\P}_{h,0} - \P_h)\hat{V}_{h+1}^\rf}(s, a)} \le c\cdot \left( \sqrt{\frac{[\hat{\V}_{h,0} (\hat{V}^{\rf}_{h+1})](s,a) \log(HSA/\delta)}{N_{h,0}(s, a)\vee 1}} + \frac{H \log(HSA/\delta)}{N_{h,0}(s, a)\vee 1} \right)
\end{align*}
where the last inequality is obtained by the empirical Bernstein inequality \citep[Theorem 4]{maurer2009empirical} with probability at least $1-\delta/(2HSA)$. Therefore
\begin{align*}
    \abs{\brac{\hat{r}_{h,0} - r_h}(s, a) + \brac{(\hat{\P}_{h,0} - \P_h) \hat{V}_{h+1}^\rf}(s, a)} \le c\cdot \left( \sqrt{\frac{[\hat{\V}_{h,0} (\hat{V}^{\rf}_{h+1})](s,a) \iota}{N_{h,0}(s, a)\vee 1}} + \frac{H \iota}{N_{h,0}(s, a)\vee 1} \right).
\end{align*}
Further taking the union bound yields the first claim over all $(h,s,a)$ with probability at least $1-\delta/4$. The claim (ii) also follows from the similar argument. Claims (iii) and (iv) can be obtained directly from Lemma~\ref{lemma:binomial-concentration} and a union bound, in a similar fashion as in the proof of Lemma~\ref{lem:concentration_hoeffding}. (Note that $N_{h,0}(s,a)\sim\Bin(n_0, d^\mu_h(s,a))$ and $N_{h,1}\sim \Bin(n_{1,h}, d^\mu_h(s,a))$ where $n_0=n/3$ and $n_{1,h}=n/(3H)$ due to our data splitting schedule.) This completes the proof.
\end{proof}

\begin{lemma}[Monotonicity for~\pevi]
\label{lem:algo1monot}
Let $\hat\pi$ be the output policy of Algorithm~\ref{algorithm:pevi-adv}. Then, on the event $\mc{E}$ defined in Lemma~\ref{lem:concentration_advref}, we have
\begin{align*}
  \hat V_h(s) \leq V^{\what{\pi}}_h(s) \leq V^\star_h(s)
\end{align*}
for any $s \in \mc{S}$ and $h \in [H]$.
\end{lemma}
\begin{proof}[Proof of Lemma~\ref{lem:algo1monot}]
We provide the proof by induction. For $h = H$, and any $s,a$,
\begin{align}
&~ Q^{\what{\pi}}_H(s,a) - \hat Q_H(s,a)
\\
\label{eq:adv_H}
= &~ r_H(s,a) - \what{r}_{H,0}(s,a) + b_{H,0}(s,a). 
\end{align}
\Eqref{eq:adv_H} is non-negative the concentration event $\mc{E}$(i) by Lemma~\ref{lem:concentration_advref}. Thus, we have $Q^{\what{\pi}}_H(s,a) - \hat Q_H(s,a) \geq 0$ for any $s,a$, and then
\begin{align}
&~ V^{\what{\pi}}_H(s) - \hat V_H(s)
\\
= &~ Q^{\what{\pi}}_H(s,\what{\pi}_H(s)) - \brac{\max_a \hat{Q}_H(s, a)} \vee 0
\\
= &~ Q^{\what{\pi}}_H(s,a_{s,H}) - \brac{\hat{Q}_H(s, a_{s,H})} \vee 0
\tag{$a_{s,H} \coloneqq \what{\pi}_H(s)$}
\\
= &~ \brac{Q^{\what{\pi}}_H(s,a_{s,H}) - \hat{Q}_H(s, a_{s,H})} \wedge Q^{\what{\pi}}_H(s,a_{s,H})
\\
\geq &~ 0,
\end{align}
where the last inequality follows from the result of \Eqref{eq:adv_H} is positive, and $Q^{\pi}_H(s,a) \in [0,1]$ for any $\pi,s \in \mc{S},a \in \mc{A}$.

We now show that, if $\hat V_{h + 1}(s) \leq V^{\what{\pi}}_{h + 1}(s)$ holds for any $s$, we also have $\hat V_{h}(s) \leq V^{\what{\pi}}_{h}(s)$ for any $s$. Recall that $\hat{V}_h(s) = \max_a \hat{Q}_h(s, a) \vee 0 = \hat{Q}_h(s, \hat{\pi}_h(s))\vee 0$. The claim clearly holds in the trivial case of $\hat{V}_h(s) = 0$. Otherwise, we have
\begin{align}
&~ V^{\what{\pi}}_{h}(s) - \hat V_{h}(s)
\\
= &~ Q^{\what{\pi}}_{h}(s,a_{s,h}) - \hat Q_{h}(s,a_{s,h})
\tag{$a_{s,h} \coloneqq \what{\pi}_h(s)$}
\\
= &~ r_h(s,a_{s,h}) + \brac{\P_{h} V_{h+1}^{\what{\pi}}}(s,a_{s,h}) - \what{r}_{h,0}(s,a_{s,h}) - \brac{\what{\P}_{h,0} \hat V_{h+1}^\rf}(s,a_{s,h})
\\
&~ - \brac{\what{\P}_{h} (\hat V_{h+1} - \hat V_{h+1}^\rf)}(s,a_{s,h}) + b_{h,0}(s,a_{s,h}) + b_{h,1}(s,a_{s,h})
\\
= &~ r_h(s,a_{s,h}) - \what{r}_{h,0}(s,a_{s,h}) + \brac{\P_{h} \left(V_{h+1}^{\what{\pi}} - \hat V_{h+1} \right)}(s,a_{s,h}) + \brac{\left(\P_{h} - \what{\P}_{h,0}\right) \hat V_{h+1}^\rf}(s,a_{s,h})
\\
&~ + \brac{\left(\P_{h} - \what{\P}_{h}\right) \left(\hat V_{h+1} - \hat V_{h+1}^\rf\right)}(s,a_{s,h}) + b_{h,0}(s,a_{s,h}) + b_{h,1}(s,a_{s,h})
\\
\label{eq:adv_h}
\geq &~ r_h(s,a_{s,h}) - \what{r}_{h,0}(s,a_{s,h})+ \brac{\left(\P_{h} - \what{\P}_{h,0}\right) \hat V_{h+1}^\rf}(s,a_{s,h})
\\
&~ + \brac{\left(\P_{h} - \what{\P}_{h}\right) \left(\hat V_{h+1} - \hat V_{h+1}^\rf\right)}(s,a_{s,h}) + b_{h,0}(s,a_{s,h}) + b_{h,1}(s,a_{s,h})
\end{align}
where the last inequality follows from $V^{\what{\pi}}_{h + 1}(s) \geq \hat V_{h + 1}(s)$ for any $s$. Note that~\Eqref{eq:adv_h} is also non-negative under $\mc{E}$(i \& ii) by Lemma~\ref{lem:concentration_advref}.
This completes the proof of $V_h(s) \leq V^{\what{\pi}}_h(s), ~ \forall s \in \mc{S}, h \in [H]$.
\end{proof}

\begin{lemma}[Performance decomposition for \pevi]
\label{lem:Vstarbd_adv}
On the event $\mc{E}$ defined in Lemma~\ref{lem:concentration_advref}, we have
\begin{align}
\sum_{s\in\mc{S}} d^{\pi_\star}_h(s) (V^\star_h(s) - \hat{V}_h(s)) \leq 2 \sum_{h' = h}^{H} \sum_{(s,a) \in \mc{S} \times \mc{A}} d_{h'}^{\pi_\star}(s,a) (b_{h',0}(s,a) + b_{h',1}(s,a)),
\end{align}
for any $h \in [H]$. 
\end{lemma}
\begin{proof}[Proof of Lemma~\ref{lem:Vstarbd_adv}]
Throughout this proof we let $\hat{Q}_h(s, \pi_\star)\defeq \hat{Q}_h(s, \pi_{\star, h}(s))$ for shorthand. On the event $\mc{E}$, we have
\begin{align}
&~ \sum_{s\in\mc{S}} d^{\pi_\star}_h(s) (V^\star_h(s) - \hat{V}_h(s)) \\
\leq &~ \sum_{s\in\mc{S}} d^{\pi_\star}_h(s) (V^\star_h(s) - \max_a \hat{Q}_h(s, a))
    \tag{by definition of $\hat{V}_h(s) =\max_a \hat{Q}_h(s, a)\vee 0$}
\\
\leq &~ \sum_{s\in\mc{S}} d^{\pi_\star}_h(s) (V^\star_h(s) - \hat{Q}_h(s, \pi_\star))
\\
= &~ \sum_{h' = h}^{H} \sum_{(s,a) \in \mc{S} \times \mc{A}} d_{h'}^{\pi_\star}(s,a) r_{h'}(s,a) - \sum_{s\in\mc{S}} d_h^{\pi_\star}(s) \hat Q_h(s_1,\pi_\star)
\\
= &~ \sum_{h' = h}^{H} \sum_{(s,a) \in \mc{S} \times \mc{A}} d_{h'}^{\pi_\star}(s,a) r_{h'}(s,a) - \sum_{h' = h}^{H}\sum_{s,a,s',a'}\left( d_{h'}^{\pi_\star}(s,a) \hat Q_{h'}(s,a) - d_{h' + 1}^{\pi_\star}(s',a') \hat Q_{h' + 1}(s',a') \right)
\\
\leq &~ \sum_{h' = h}^{H} \sum_{(s,a) \in \mc{S} \times \mc{A}} d_{h'}^{\pi_\star}(s,a) r_{h'}(s,a) - \sum_{h' = h}^{H}\sum_{s,a,s'}\left( d_{h'}^{\pi_\star}(s,a) \hat Q_{h'}(s,a) - d_{h' + 1}^{\pi_\star}(s') \hat V_{h' + 1}(s') \right)
\tag{by definition of $\hat{V}_{h'}$}
\\
= &~ \sum_{h' = h}^{H} \sum_{(s,a) \in \mc{S} \times \mc{A}} d_{h'}^{\pi_\star}(s,a) r_h(s,a) - \sum_{h' = h}^{H}\sum_{(s,a) \in \mc{S} \times \mc{A}}d_{h'}^{\pi_\star}(s,a) \left( \hat Q_{h'}(s,a) - \brac{\P_{h'} \hat V_{h' + 1}}(s,a) \right)
\\
= &~ \sum_{h' = h}^{H} \sum_{(s,a) \in \mc{S} \times \mc{A}} d_{h'}^{\pi_\star}(s,a) \left( r_{h'}(s,a) + \brac{\P_{h'} \hat V_{h' + 1}}(s,a) - \hat Q_{h'}(s,a) \right)
\\
= &~ \sum_{h' = h}^{H} \sum_{(s,a) \in \mc{S} \times \mc{A}} d_{h'}^{\pi_\star}(s,a) \bigg( r_{h'}(s,a) - \what{r}_{h',0}(s,a) + \brac{\P_{h'} \hat V_{h' + 1}}(s,a) - \brac{\what{\P}_{h',0} \hat{V}^{\rf}_{h'+1}}(s,a)
\\
&~ - \brac{\what{\P}_{h',1} (\hat{V}_{h'+1} - \hat{V}^{\rf}_{h'+1})}(s,a) + b_{h',0}(s,a) + b_{h',1}(s,a)\bigg)
\\
\leq &~ 2 \sum_{h' = h}^{H} \sum_{(s,a) \in \mc{S} \times \mc{A}} d_{h'}^{\pi_\star}(s,a) (b_{h',0}(s,a) + b_{h',1}(s,a)).
\tag{by Lemma~\ref{lem:concentration_advref}}
\end{align}
This completes the proof.
\end{proof}

\begin{lemma}[Total variance lemma]
On the event $\mc{E}$ defined in Lemma~\ref{lem:concentration_advref}, the reference value function $\hat{V}^{\rf}$ obtained in Algorithm~\ref{algorithm:pevi-adv} satisfies
\label{lemma:total-variance}
\begin{align}
\sum_{h = 1}^{H} \sum_{(s,a) \in \mc{S} \times \mc{A}} d_h^{\pi_\star}(s,a) \left[\V_{h} \hat{V}^{\rf}_{h+1}\right](s,a) \leq H^2 + c\sqrt{\frac{H^9SC^\star\iota^2}{n_{\rf}}},
\end{align}
where $c$ is an absolute constant.
\end{lemma}
\begin{proof}[Proof of Lemma~\ref{lemma:total-variance}]
We first decompose our target as follows,
\begin{align}
&~ \sum_{h = 1}^{H} \sum_{(s,a) \in \mc{S} \times \mc{A}} d_h^{\pi_\star}(s,a) [\V_{h} (\hat{V}^{\rf}_{h+1})](s,a)
\\
= &~ \underbrace{\sum_{h = 1}^{H} \sum_{(s,a) \in \mc{S} \times \mc{A}} d_h^{\pi_\star}(s,a) [\V_{h} V^\star_{h + 1}](s,a)}_{\text{(I)}} + \underbrace{\sum_{h = 1}^{H} \sum_{(s,a) \in \mc{S} \times \mc{A}} d_h^{\pi_\star}(s,a) [\V_{h} \hat{V}^{\rf}_{h+1} - \V_{h} V^\star_{h + 1}](s,a)}_{\text{(II)}}.
\end{align}
We now bound term (I) and term (II) separately.

Let $\mc{F}_{h+1}$ denote the $\sigma$-algebra that contains all information about the trajectory up to $s_{h+1}$ but not $a_{h+1}$. We have
\begin{align}
\text{(I)} = &~ \sum_{h = 1}^{H} \E_{d^{\pi_\star}} \left[\Var\left[  V^\star_{h + 1}(s_{h + 1}) \middle| s_h, a_h \right]\right]
\\
= &~ \sum_{h = 1}^{H} \E_{d^{\pi_\star}} \left[\E\left[  \left(V^\star_{h + 1}(s_{h + 1}) + r_h(s_h,a_h) -  V^\star_h(s_h)\right)^2 \middle| s_h, a_h \right]\right]
\\
\tag{by $\E[ V^\star_{h + 1}(s_{h + 1}) | s_h, a_h] = V^\star_h(s_h) - r_h(s_h,a_h)$}
\\
= &~ \sum_{h = 1}^{H} \E_{d^{\pi_\star}} \left[ \left(V^\star_{h + 1}(s_{h + 1}) + r_h(s_h,a_h) -  V^\star_h(s_h)\right)^2 \right]
\\
= &~ \sum_{h = 1}^{H} \E_{d^{\pi_\star}} \left[ \left(V^\star_{h + 1}(s_{h + 1}) + r_h(s_h,a_h) -  V^\star_h(s_h)\right)^2 \right]
\\
&~ + 2 \underbrace{\sum_{1 \leq h < h' \leq H} \E_{d^{\pi_\star}} \left[ \left(V^\star_{h + 1}(s_{h + 1}) + r_h(s_h,a_h) -  V^\star_h(s_h)\right) \cdot \left( V^\star_{h' + 1}(s_{h' + 1}) + r(s_h',a_h') -  V^\star_h(s_h')\right) \right]}_{=0,\text{ because $\left(V^\star_{h + 1}(s_{h + 1}) + r_h(s_h,a_h) -  V^\star_h(s_h)\right)\E_{d^{\pi_\star}}[V^\star_{h' + 1}(s_{h' + 1}) -  V^\star_{h'}(s_{h'}) + r_{h'}(s_{h'},a_{h'}) | \mc{F}_{h+1}] = 0$ for any $h<h'$}}
\\
= &~ \E_{d^{\pi_\star}} \left[ \left(\sum_{h = 1}^{H} \left(V^\star_{h + 1}(s_{h + 1}) + r_h(s_h,a_h) -  V^\star_h(s_h)\right)\right)^2 \right]
\\
= &~ \E_{d^{\pi_\star}} \left[ \left(\sum_{h = 1}^{H} r_h(s_h,a_h) + \sum_{h = 1}^{H} \left(V^\star_{h + 1}(s_{h + 1}) -  V^\star_h(s_h)\right)\right)^2 \right]
\\
= &~ \E_{d^{\pi_\star}} \left[ \left(\sum_{h = 1}^{H} r_h(s_h,a_h)  -  V^\star_1(s_1)\right)^2 \right]
\\
= &~ \Var_{d^{\pi_\star}}\paren{ \sum_{h=1}^H r_h(s_h, a_h) } \le H^2.
\end{align}
For (II),
\begin{align}
\text{(II)} = &~ \sum_{h = 1}^{H} \sum_{(s,a) \in \mc{S} \times \mc{A}} d_h^{\pi_\star}(s,a) [\V_{h} \hat{V}^{\rf}_{h+1} - \V_{h} V^\star_{h + 1}](s,a)
\\
= &~ \sum_{h = 1}^{H} \sum_{(s,a) \in \mc{S} \times \mc{A}} d_h^{\pi_\star}(s,a) \left[\P_{h} \left(\hat{V}^{\rf}_{h+1}\right)^2 - \left(\P_{h} \hat{V}^{\rf}_{h+1}\right)^2 - \P_{h} \left(V^\star_{h + 1}\right)^2 + \left(\P_{h} V^\star_{h + 1}\right)^2\right](s,a)
\\
\leq &~ \sum_{h = 1}^{H} \sum_{(s,a) \in \mc{S} \times \mc{A}} d_h^{\pi_\star}(s,a) \big[\big|\P_{h} (\hat{V}^{\rf}_{h+1} + V^\star_{h + 1}) (\hat{V}^{\rf}_{h+1} - V^\star_{h + 1})\big| 
\\
&~ + \big|\P_{h} (\hat{V}^{\rf}_{h+1} + V^\star_{h + 1}) \P_{h} (\hat{V}^{\rf}_{h+1} - V^\star_{h + 1})\big|\big](s,a)
\\
\leq &~ 4H \sum_{h = 1}^{H} \sum_{(s,a) \in \mc{S} \times \mc{A}} d_h^{\pi_\star}(s,a) \left[\left|\P_{h} (\hat{V}^{\rf}_{h+1} - V^\star_{h + 1})\right|\right](s,a)
\\
\leq &~ 4H \sum_{h = 1}^{H} \sum_{(s,a) \in \mc{S} \times \mc{A}} d_{h + 1}^{\pi_\star}(s') \left[V^\star_{h + 1} - \hat{V}^{\rf}_{h+1}\right](s')
\tag{As $\hat{V}^{\rf}_{h+1}\le V^\star_{h+1}$}
\\
\leq &~ 4H^2 \max_{h \in [H]} \sum_{(s,a) \in \mc{S} \times \mc{A}} d_{h}^{\pi_\star}(s) \left(V^\star_{h} - \hat{V}^{\rf}_{h}\right)(s)
\\
\leq &~ c\sqrt{\frac{H^9SC^\star\iota^2}{n_{\rf}}},
\end{align}
where the last inequality follows from Lemma~\ref{lemma:vref-bound}. Combining (I) and (II), we complete the proof.
\end{proof}

\begin{lemma}[Guarantees for $\hat{V}^{\rf}$]
  \label{lemma:vref-bound}
  On the event $\mc{E}$ defined in Lemma~\ref{lem:concentration_advref}, the reference value function $\hat{V}^{\rf}$ obtained in Algorithm~\ref{algorithm:pevi-adv} satisfies 
  \begin{align*}
    \max_{h\in[H]} \sum_{s\in\mc{S}} d^{\pi_\star}_h(s) \paren{ V_h^\star(s) - \hat{V}_h^{\rf}(s) } \le c\sqrt{\frac{H^5SC^\star\iota^2}{n_{\rf}}},
  \end{align*}
  where $c>0$ is an absolute constant.
\end{lemma}
\begin{proof}[Proof of Lemma~\ref{lemma:vref-bound}]
  This is a direct corollary of Theorem~\ref{theorem:vilcb} (cf. the end of its proof in Section~\ref{appendix:proof-vilcb-main}).
\end{proof}

\subsection{Proof of the Main Theorem}
\label{sec:provemainstep3}

We now provide the proof of Theorem~\ref{theorem:main}. We assume we are on the good event $\mc{E}$ defined in Lemma~\ref{lem:concentration_advref}, which happens with probability at least $1-\delta$. By Lemma~\ref{lem:Vstarbd_adv}, we know
\begin{align}
\label{eq:b0b1}
&~ \sum_{s\in\mc{S}} d^{\pi_\star}_h(s) (V^\star_h(s) - \hat{V}_h(s)) 
\\
\leq &~ 2 \underbrace{\sum_{h = 1}^{H} \sum_{(s,a) \in \mc{S} \times \mc{A}} d_h^{\pi_\star}(s,a) b_{h,0}(s,a)}_{\text{(I)}} + 2 \underbrace{\sum_{h = 1}^{H} \sum_{(s,a) \in \mc{S} \times \mc{A}} d_h^{\pi_\star}(s,a) b_{h,1}(s,a)}_{\text{(II)}},
\end{align}
for any $h \in [H]$.

We first study the term (I) of \Eqref{eq:b0b1}. Observe that for any $(s,a)$,
\begin{align}
&~ [\hat{\V}_{h,0} \hat{V}^{\rf}_{h+1}](s,a) - [\V_{h} \hat{V}^{\rf}_{h+1}](s,a)
\\
= &~ \hat{\P}_{h,0} (\hat{V}^{\rf}_{h+1})^2 - (\hat{\P}_{h,0} \hat{V}^{\rf}_{h+1})^2 - (\P_{h} (\hat{V}^{\rf}_{h+1})^2 - (\P_{h} \hat{V}^{\rf}_{h+1})^2)
\tag{``$(s,a)$'' is omitted}
\\
\leq &~ |(\hat{\P}_{h,0} - \P_{h})(\hat{V}^{\rf}_{h+1})^2| + |(\hat{\P}_{h,0} + \P_{h})\hat{V}^{\rf}_{h+1} \cdot (\hat{\P}_{h,0} - \P_{h})\hat{V}^{\rf}_{h+1}|
\\
\leq &~ |(\hat{\P}_{h,0} - \P_{h})(\hat{V}^{\rf}_{h+1})^2| + 2 H |(\hat{\P}_{h,0} - \P_{h})\hat{V}^{\rf}_{h+1}|
\tag{$|\hat{V}^{\rf}|\le H$}
\\
\label{eq:spvar2popvar}
\leq &~  c \sqrt{\frac{H^4 \iota}{N_0(s,a) \vee 1}}.
\end{align}
where the last inequality follows from the Azuma-Hoeffding inequality and the fact that $\hat{V}^{\rf}_{h+1}(s)$ is obtained from data independent of $\what{P}_{h,0}$.


Thus, we obtain the following bound for term (I) 
\begin{align}
\text{(I)} = &~ \sum_{h = 1}^{H} \sum_{(s,a) \in \mc{S} \times \mc{A}} d_h^{\pi_\star}(s,a) b_{h,0}(s,a)
\\
= &~ c \sum_{h = 1}^{H} \sum_{(s,a) \in \mc{S} \times \mc{A}} d_h^{\pi_\star}(s,a) \left(\sqrt{\frac{ [\hat{\V}_{h,0} (\hat{V}^{\rf}_{h+1})](s,a) \iota}{N_{h,0}(s,a)\vee 1}} + \frac{H \iota}{N_{h,0}(s, a)\vee 1}\right)
\\
\leq &~ c \sum_{h = 1}^{H} \sum_{(s,a) \in \mc{S} \times \mc{A}} d_h^{\pi_\star}(s,a) \left(\sqrt{\frac{ [\V_{h} (\hat{V}^{\rf}_{h+1})](s,a) \iota + \sqrt{\frac{H^4 \iota}{N_0(s,a) \vee 1}}}{N_{h,0}(s,a)\vee 1}} + \frac{H \iota}{N_{h,0}(s, a)\vee 1}\right)
\tag{by \Eqref{eq:spvar2popvar}}
\\
\leq &~ c \sum_{h = 1}^{H} \sum_{(s,a) \in \mc{S} \times \mc{A}} d_h^{\pi_\star}(s,a) \left(\sqrt{\frac{ [\V_{h} (\hat{V}^{\rf}_{h+1})](s,a) \iota}{N_{h,0}(s,a)\vee 1}} + \frac{H \iota^{\nicefrac{1}{4}}}{(N_{h,0}(s,a)\vee 1)^{\nicefrac{3}{4}}} + \frac{H \iota}{N_{h,0}(s, a)\vee 1}\right)
\\
\leq &~ c \sum_{h = 1}^{H} \sum_{(s,a) \in \mc{S} \times \mc{A}} d_h^{\pi_\star}(s,a) \left(\sqrt{\frac{ [\V_{h} (\hat{V}^{\rf}_{h+1})](s,a) \iota}{N_{h,0}(s,a)\vee 1}} + \sqrt{\frac{1}{N_{h,0}(s,a)\vee 1}} + \frac{H^2 \iota^{\nicefrac{1}{2}} + H \iota}{N_{h,0}(s,a)\vee 1}\right)
\\
\label{eq:b0terms}
= &~ c \underbrace{\sum_{h = 1}^{H} \sum_{(s,a) \in \mc{S} \times \mc{A}} d_h^{\pi_\star}(s,a) \sqrt{\frac{ [\V_{h} (\hat{V}^{\rf}_{h+1})](s,a) \iota}{N_{h,0}(s,a)\vee 1}}}_{\text{(I.a)}} + c \underbrace{\sum_{h = 1}^{H} \sum_{(s,a) \in \mc{S} \times \mc{A}} d_h^{\pi_\star}(s,a) \sqrt{\frac{1}{N_{h,0}(s,a)\vee 1}}}_{\text{(I.b)}}
\\
&~ + c \underbrace{\sum_{h = 1}^{H} \sum_{(s,a) \in \mc{S} \times \mc{A}} d_h^{\pi_\star}(s,a) \frac{H^2 \iota^{\nicefrac{1}{2}} + H \iota}{N_{h,0}(s,a)\vee 1}}_{\text{(I.c)}}
\end{align}
where the last inequality follows from Cauchy–Schwarz inequality.

We now discuss the three terms in \Eqref{eq:b0terms} separately:
\begin{align}
\text{(I.a)} = &~ \sum_{h = 1}^{H} \sum_{(s,a) \in \mc{S} \times \mc{A}} d_h^{\pi_\star}(s,a) \sqrt{\frac{ [\V_{h} (\hat{V}^{\rf}_{h+1})](s,a) \iota}{N_{h,0}(s,a)\vee 1}}
\\
\leq &~ \sum_{h = 1}^{H} \sum_{(s,a) \in \mc{S} \times \mc{A}} d_h^{\pi_\star}(s,a) \cdot \sqrt{\frac{[\V_{h} (\hat{V}^{\rf}_{h+1})](s,a) \iota^2}{n_0 d_h^\mu(s,a)}}
\tag{by the concentration event $\mc{E}$(iii)}
\\
\leq &~ \sqrt{\frac{C^\star \iota^2}{n_0}} \sum_{h = 1}^{H} \sum_{(s,a) \in \mc{S} \times \mc{A}} \sqrt{d_h^{\pi_\star}(s,a) [\V_{h} (\hat{V}^{\rf}_{h+1})](s,a)}
\\
= &~ \sqrt{\frac{C^\star \iota^2}{n_0}} \sum_{h = 1}^{H} \sum_{(s,a) \in \mc{S} \times \mc{A}} \sqrt{\1\{a = \pi_\star(s)\} d_h^{\pi_\star}(s,a) [\V_{h} (\hat{V}^{\rf}_{h+1})](s,a)}
\\
\leq &~ \sqrt{\frac{C^\star \iota^2}{n_0}}  \sqrt{\sum_{h = 1}^{H} \sum_{(s,a) \in \mc{S} \times \mc{A}} \1\{a = \pi_\star(s)\}} \cdot \sqrt{\sum_{h = 1}^{H} \sum_{(s,a) \in \mc{S} \times \mc{A}} d_h^{\pi_\star}(s,a) [\V_{h} (\hat{V}^{\rf}_{h+1})](s,a)}
\\
\tag{by Cauchy–Schwarz inequality}
\\
\label{eq:b0term1_temp}
\leq &~ \sqrt{\frac{HS C^\star \iota^2}{n_0}}  \sqrt{\sum_{h = 1}^{H} \sum_{(s,a) \in \mc{S} \times \mc{A}} d_h^{\pi_\star}(s,a) [\V_{h} (\hat{V}^{\rf}_{h+1})](s,a)}
\\
\label{eq:b0term1}
\leq &~ \sqrt{\frac{H S C^\star \iota^2}{n_0}} \sqrt{H^2 + c\sqrt{\frac{H^9SC^\star\iota^2}{n_{\rf}}}}
\\
\tag{follows from the total variance lemma (Lemma~\ref{lemma:total-variance})}
\\
\leq &~ \sqrt{\frac{H^3 S C^\star \iota^2}{n_0}} + c \sqrt{\frac{H S C^\star \iota^2}{n_0}} \sqrt[4]{\frac{H^9SC^\star\iota^2}{n_{\rf}}}
\\
\leq &~ \sqrt{\frac{H^3 S C^\star \iota^2}{n_0}} + c \left( \sqrt{\frac{H^3 S C^\star \iota^2}{n_0}} + \frac{H^4 S C^\star \iota^2}{\sqrt{n_0 n_\rf}}\right)
\tag{by $\sqrt{ab}\le (a+b)/2$}
\\
\leq &~ c \left(\sqrt{\frac{H^3 S C^\star \iota^2}{n_0}} + \frac{H^4 S C^\star \iota^2}{\sqrt{n_0 n_\rf}} \right).
\end{align}

Term (I.b) is a smaller-order term compared with (I.a):
\begin{align}
\text{(I.b)} = &~ \sum_{h = 1}^{H} \sum_{(s,a) \in \mc{S} \times \mc{A}} d_h^{\pi_\star}(s,a) \sqrt{\frac{1}{N_{h,0}(s,a)\vee 1}}
\\
\leq &~ \sum_{h = 1}^{H} \sum_{(s,a) \in \mc{S} \times \mc{A}} d_h^{\pi_\star}(s,a) \cdot \sqrt{\frac{\iota}{n_0 d_h^\mu(s,a)}}
\tag{by the concentration event $\mc{E}$(iii)}
\\
\leq &~ \sqrt{\frac{C^\star \iota}{n_0}} \sum_{h = 1}^{H} \sum_{(s,a) \in \mc{S} \times \mc{A}} \sqrt{d_h^{\pi_\star}(s,a)}
\\
= &~ \sqrt{\frac{C^\star \iota}{n_0}} \sum_{h = 1}^{H} \sum_{(s,a) \in \mc{S} \times \mc{A}} \sqrt{\1\{a = \pi_\star(s)\} d_h^{\pi_\star}(s,a)}
\\
\leq &~ \sqrt{\frac{C^\star \iota}{n_0}}  \sqrt{\sum_{h = 1}^{H} \sum_{(s,a) \in \mc{S} \times \mc{A}} \1\{a = \pi_\star(s)\}} \cdot \sqrt{\sum_{h = 1}^{H} \sum_{(s,a) \in \mc{S} \times \mc{A}} d_h^{\pi_\star}(s,a)}
\\
\tag{by Cauchy–Schwarz inequality}
\\
\label{eq:b0term2}
\leq &~ \sqrt{\frac{H^2 C^\star \iota}{n_0}}.
\end{align}

Finally, term (I.c)
\begin{align}
\text{(I.c)} = &~ \sum_{h = 1}^{H} \sum_{(s,a) \in \mc{S} \times \mc{A}} d_h^{\pi_\star}(s,a) \frac{H^2 \iota^{\nicefrac{1}{2}} + H \iota}{N_{h,0}(s,a)\vee 1}
\\
\leq &~ \sum_{h = 1}^{H} \sum_{(s,a) \in \mc{S} \times \mc{A}} d_h^{\pi_\star}(s,a) \cdot \frac{H^2 \iota^{\nicefrac{3}{2}} + H \iota^2}{n_0 d_h^\mu(s,a)}
\tag{by the concentration event $\mc{E}$(iii)}
\\
\label{eq:b0term3}
\leq &~ \frac{H^3 S C^\star \iota^{\nicefrac{3}{2}} + H^2 S C^\star \iota^2}{n_0}.
\end{align}

Substituting \Eqref{eq:b0term1}, \Eqref{eq:b0term2}, and \Eqref{eq:b0term3} into \Eqref{eq:b0terms}, we obtain
\begin{align}
\text{(I)} = &~ \sum_{h = 1}^{H} \sum_{(s,a) \in \mc{S} \times \mc{A}} d_h^{\pi_\star}(s,a) b_{h,0}(s,a)
\\
\label{eq:term1bound}
\leq &~ c \cdot \left( \sqrt{\frac{H^3 S C^\star \iota^2}{n_0}} + \frac{H^4 S C^\star \iota^2}{\sqrt{n_0 n_\rf}} + \frac{H^3 S C^\star \iota^{\nicefrac{3}{2}} + H^2 S C^\star \iota^2}{n_0} \right).
\end{align}

We now study the term (II) of \Eqref{eq:b0b1}. Let $g_{h+1} \coloneqq \hat{V}_{h+1} - \hat{V}^{\rf}_{h+1}$ (which by our data splitting only depends on the datasets $\mc{D}_{\rf},\mc{D}_0$ and is independent of $\mc{D}_{1,h}$). By a similar argument as \Eqref{eq:spvar2popvar}, we have
\begin{align}
 [\hat{\V}_{h,1} g](s,a) - [\V_{h} g](s,a) \leq c \sqrt{\frac{H^4 \iota}{N_1(s,a) \vee 1}}.
\end{align}
for any $(s,a)$. Thus,
\begin{align}
\text{(II)} = &~ \sum_{h = 1}^{H} \sum_{(s,a) \in \mc{S} \times \mc{A}} d_h^{\pi_\star}(s,a) b_{h,1}(s,a)
\\
= &~ c \sum_{h = 1}^{H} \sum_{(s,a) \in \mc{S} \times \mc{A}} d_h^{\pi_\star}(s,a) \left(\sqrt{\frac{ \brac{\hat\V_{h} (\hat{V}_{h+1} - \hat{V}^{\rf}_{h+1})}(s,a) \iota}{N_{h,1}(s,a)\vee 1}} + \frac{H \iota}{N_{h,0}(s, a)\vee 1} \right)
\\
\leq &~ c \sum_{h = 1}^{H} \sum_{(s,a) \in \mc{S} \times \mc{A}} d_h^{\pi_\star}(s,a) \left(\sqrt{\frac{ \brac{\V_{h} (\hat{V}_{h+1} - \hat{V}^{\rf}_{h+1})}(s,a) \iota + \sqrt{\frac{H^4 \iota}{N_1(s,a) \vee 1}}}{N_{h,1}(s,a)\vee 1}} + \frac{H \iota}{N_{h,0}(s, a)\vee 1}\right)
\\
\leq &~ c \sum_{h = 1}^{H} \sum_{(s,a) \in \mc{S} \times \mc{A}} d_h^{\pi_\star}(s,a) \left(\sqrt{\frac{ \brac{\V_{h} (\hat{V}_{h+1} - \hat{V}^{\rf}_{h+1})}(s,a) \iota}{N_{h,1}(s,a)\vee 1}} + \sqrt{\frac{1}{N_{h,1}(s,a)\vee 1}} + \frac{H^2 \iota^{\nicefrac{1}{2}} + H \iota}{N_{h,1}(s,a)\vee 1}\right)\\
\tag{by a similar argument of \Eqref{eq:b0terms}}
\\
\label{eq:b1terms}
= &~ c \underbrace{\sum_{h = 1}^{H} \sum_{(s,a) \in \mc{S} \times \mc{A}} d_h^{\pi_\star}(s,a) \sqrt{\frac{ \brac{\V_{h} (\hat{V}_{h+1} - \hat{V}^{\rf}_{h+1})}(s,a) \iota}{N_{h,1}(s,a)\vee 1}}}_{\text{(II.a)}}
\\
&~ + c \underbrace{\sum_{h = 1}^{H} \sum_{(s,a) \in \mc{S} \times \mc{A}} d_h^{\pi_\star}(s,a) \sqrt{\frac{1}{N_{h,1}(s,a)\vee 1}}}_{\text{(II.b)}} + c \underbrace{\sum_{h = 1}^{H} \sum_{(s,a) \in \mc{S} \times \mc{A}} d_h^{\pi_\star}(s,a) \frac{H^2 \iota^{\nicefrac{1}{2}} + H \iota}{N_{h,1}(s,a)\vee 1}}_{\text{(II.c)}}.
\end{align}

Now, we also bound (II.a), (II.b) and (II.c) separately:
\begin{align}
\text{(II.a)} = &~ \sum_{h = 1}^{H} \sum_{(s,a) \in \mc{S} \times \mc{A}} d_h^{\pi_\star}(s,a) \sqrt{\frac{ \brac{\V_{h} (\hat{V}_{h+1} - \hat{V}^{\rf}_{h+1})}(s,a) \iota}{N_{h,1}(s,a)\vee 1}}
\\
\leq &~ \sum_{h = 1}^{H} \sum_{(s,a) \in \mc{S} \times \mc{A}} d_h^{\pi_\star}(s,a) \cdot \sqrt{\frac{H \brac{\V_{h} (\hat{V}_{h+1} - \hat{V}^{\rf}_{h+1})}(s,a) \iota^2}{n_1 d_h^\mu(s,a)}}
\tag{by the concentration event $\mc{E}$(iv)}
\\
\leq &~ \sqrt{\frac{H C^\star \iota^2}{n_1}} \sum_{h = 1}^{H} \sum_{(s,a) \in \mc{S} \times \mc{A}} \sqrt{d_h^{\pi_\star}(s,a) \brac{\V_{h} (\hat{V}_{h+1} - \hat{V}^{\rf}_{h+1})}(s,a) }
\\
= &~ \sqrt{\frac{H C^\star \iota^2}{n_1}} \sum_{h = 1}^{H} \sum_{s} \sqrt{d_h^{\pi_\star}(s) \brac{\V_{h} (\hat{V}_{h+1} - \hat{V}^{\rf}_{h+1})}(s,\pi_{\star,h}(s))}
\\
\leq &~ \sqrt{\frac{H^2 S C^\star \iota^2}{n_1}} \sum_{h = 1}^{H} \sum_{s} d_h^{\pi_\star}(s) \brac{\V_{h} (\hat{V}_{h+1} - \hat{V}^{\rf}_{h+1})}(s,\pi_{\star,h}(s))
\\
\tag{by Cauchy–Schwarz inequality, e.g., similar to \Eqref{eq:b0term1_temp}}
\\
\leq &~ \sqrt{\frac{H^2 S C^\star \iota^2}{n_1}} \sum_{h = 1}^{H} \sum_{s} d_h^{\pi_\star}(s) \brac{\P_{h} (\hat{V}_{h+1} - \hat{V}^{\rf}_{h+1})^2}(s,\pi_{\star,h}(s))
\\
\leq &~ \sqrt{\frac{H^4 S C^\star \iota^2}{n_1}} \max_{h \in \{1,2,\dotsc,H\}} \sum_{s} d_h^{\pi_\star}(s) \brac{\P_{h} (\hat{V}_{h+1} - \hat{V}^{\rf}_{h+1})^2}(s,\pi_{\star,h}(s))
\\
\leq &~ \sqrt{\frac{H^4 S C^\star \iota^2}{n_1}} \max_{h \in \{1,2,\dotsc,H\}} \sum_{s} d_h^{\pi_\star}(s) \brac{\P_{h} (V^\star_{h + 1} - \hat{V}^{\rf}_{h+1})^2}(s,\pi_{\star,h}(s))
\tag{by $\hat{V}_{h+1} \leq V^\star_{h + 1}$ from Lemma~\ref{lem:algo1monot}}
\\
= &~ \sqrt{\frac{H^4 S C^\star \iota^2}{n_1}} \max_{h \in \{1,2,\dotsc,H\}} \sum_{s'} d_{h + 1}^{\pi_\star}(s') \left(V^\star_{h + 1}(s') - \hat{V}^{\rf}_{h+1}(s')\right)^2
\\
\leq &~ \sqrt{\frac{H^4 S C^\star \iota^2}{n_1}} \max_{h \in \{1,2,\dotsc,H\}} \sum_{s} d_{h}^{\pi_\star}(s) \left(V^\star_{h}(s) - \hat{V}^{\rf}_{h}(s)\right) \cdot H
\\
\label{eq:b1term1}
\leq &~ c \frac{H^{5.5} S C^\star \iota^2}{\sqrt{n_1 n_\rf}},
\end{align}
where the last inequality follows from the guarantee for $\hat{V}^{\rf}$ in Lemma~\ref{lemma:vref-bound}:
\begin{align*}
  \max_{h \in \{1,2,\dotsc,H\}} \sum_{s} d_{h}^{\pi_\star}(s) (V^\star_{h}(s) - \hat{V}^{\rf}_{h}(s)) \leq c \sqrt{\frac{H^5 S C^\star \iota^2}{n_\rf}}.
\end{align*}


By similar arguments as \Eqref{eq:b0term2} and \Eqref{eq:b0term3}, we also have the bounds on (II.b) and (II.c) as follows:
\begin{align}
\text{(II.b)} = &~ \sum_{h = 1}^{H} \sum_{(s,a) \in \mc{S} \times \mc{A}} d_h^{\pi_\star}(s,a) \sqrt{\frac{1}{N_{h,1}(s,a)\vee 1}}
\\
  \leq &~ \sum_{h = 1}^{H} \sum_{(s,a) \in \mc{S} \times \mc{A}} d_h^{\pi_\star}(s,a) \cdot \sqrt{\frac{H \iota}{n_1 d_h^\mu(s,a)}}
         \tag{by the concentration event $\mc{E}$(iv)}
\\
\leq &~ \sqrt{\frac{H C^\star \iota}{n_1}} \sum_{h = 1}^{H} \sum_{(s,a) \in \mc{S} \times \mc{A}} \sqrt{d_h^{\pi_\star}(s,a)}
\\
= &~ \sqrt{\frac{H C^\star \iota}{n_1}} \sum_{h = 1}^{H} \sum_{(s,a) \in \mc{S} \times \mc{A}} \sqrt{\1\{a = \pi_\star(s)\} d_h^{\pi_\star}(s,a)}
\\
\leq &~ \sqrt{\frac{H C^\star \iota}{n_1}}  \sqrt{\sum_{h = 1}^{H} \sum_{(s,a) \in \mc{S} \times \mc{A}} \1\{a = \pi_\star(s)\}} \cdot \sqrt{\sum_{h = 1}^{H} \sum_{(s,a) \in \mc{S} \times \mc{A}} d_h^\star(s,a)}
\\
\tag{by Cauchy–Schwarz inequality}
\\
\label{eq:b1term2}
\leq &~ \sqrt{\frac{H^3 S C^\star \iota}{n_1}}.
\end{align}

\begin{align}
\text{(II.c)} = &~ \sum_{h = 1}^{H} \sum_{(s,a) \in \mc{S} \times \mc{A}} d_h^{\pi_\star}(s,a) \frac{H^2 \iota^{\nicefrac{1}{2}} + H \iota}{N_{h,1}(s,a)\vee 1}
\\
  \leq &~ \sum_{h = 1}^{H} \sum_{(s,a) \in \mc{S} \times \mc{A}} d_h^{\pi_\star}(s,a) \cdot \frac{H^3 \iota^{\nicefrac{3}{2}} + H^2 \iota^2}{n_1 d_h^\mu(s,a)}
         \tag{by the concentration event $\mc{E}$(iv)}
\\
\label{eq:b1term3}
\leq &~ \frac{H^4 S C^\star \iota^{\nicefrac{3}{2}} + H^3 S C^\star \iota^2}{n_1}.
\end{align}

Substituting \Eqref{eq:b1term1}, \Eqref{eq:b1term2}, and \Eqref{eq:b1term3} into \Eqref{eq:b1terms}, we obtain
\begin{align}
\text{(II)} = &~ \sum_{h = 1}^{H} \sum_{(s,a) \in \mc{S} \times \mc{A}} d_h^{\pi_\star}(s,a) b_{h,1}(s,a)
\\
\label{eq:term2bound}
\leq &~ c \cdot \left(\frac{H^{5.5} S C^\star \iota^2}{\sqrt{n_1 n_\rf}} + \sqrt{\frac{H^3 S C^\star \iota}{n_1}} + \frac{H^4 S C^\star \iota^{\nicefrac{3}{2}} + H^3 S C^\star \iota^2}{n_1} \right).
\end{align}

By definition, we know $n_\rf = n_0 = n_1 = n/3$. Therefore, combining \Eqref{eq:term2bound} and \Eqref{eq:term1bound}, we obtain
\begin{align}
V^\star_1(s_1) - \hat V_1(s_1) \leq &~ 2 \cdot \text{(I)} + 2 \cdot \text{(II)}
\\
  \leq &~  c \cdot \left( \sqrt{\frac{H^3 S C^\star \iota^2}{n}} + \frac{H^{5.5} S C^\star \iota^2 }{n} \right).
\end{align}
The right-hand-side is upper bounded by $\eps$ as long as
\begin{align*}
  n \ge \wt{O}\paren{ H^3SC^\star\iota^2/\eps^2 + \frac{H^{5.5}SC^\star\iota^2}{\eps} }
\end{align*}
Finally, by the monotonicity property in Lemma~\ref{lem:algo1monot}, the above also implies the guarantee on $\hat{\pi}$:
\begin{align*}
  V_1^\star(s_1) - V^{\hat{\pi}}_1(s_1) \le c \cdot \left( \sqrt{\frac{H^3 S C^\star \iota^2}{n}} + \frac{H^{5.5} S C^\star \iota^2 }{n} \right) \le \eps.
\end{align*}
This completes the proof.

\section{Proof of Theorem~\ref{theorem:online-lower}}
\label{appendix:proof-online-lower}


To avoid notational clash, in this proof we use upper-case letters $S_h,A_h,R_h$ to denote the actual states and actions seen during the algorithm execution (which are random variables), and reserve the lower-case letters $s_i, s_g, s_b, a$ for indexing the (fixed) states and actions of the MDP.

Define an integer
\begin{align*}
  K \defeq \min\set{\floor{C^\star}, A}.
\end{align*}
By our assumption that $C^\star\ge 2$, we have $2\le K\le A$ and $K\in[2/3, 1]\cdot \min\set{C^\star, A}$. Therefore, it suffices to prove the desired performance lower bound for $n\le c_0 \cdot H^3SK/\eps^2$.


\paragraph{Construction of hard instances}
We now construct a family of MDPs with $S+2$ states, $2H+1$ steps, and $A$ actions for any $S\ge 1$ and $H\ge 1$. (This rescaling only affects $S,H$ by at most a multiplicative constant and thus does not affect our result.)

Each MDP $M_{\ba^\star}$ is indexed by a vector $\ba^\star=(a^\star_{h,i})\in[A]^{HS}$ and is specified as follows:
\begin{itemize}[wide]
\item State space: There are $S$ ``bandit states'' $\set{s_i}_{i\in[S]}$, one ``good state'' $s_g$, and one ``bad state'' $s_b$.
\item The action space is $\mc{A}\defeq [A]$.
\item Transitions:
  \begin{itemize}
  \item At each $h\in\set{1,\dots,H}$, the bandit state $s_i$ can only transition to $s_i$ itself, $s_g$, or $s_b$. The transition probabilities are
    \begin{equation*}
      \left\{
        \begin{aligned}
          & \P_h(s_i | s_i, a) = 1 - \frac{1}{H}~~~\textrm{for all}~a\in[A], \\
          & \P_h(s_g | s_i, a) = \P_h(s_b | s_i, a) = \frac{1}{2H}~~~\textrm{for all}~~a\neq a^\star_{h,i}, \\
          & \P_h(s_g | s_i, a^\star_{h,i}) = \frac{1}{H}\paren{\frac{1}{2} + \tau},~~~\P_h(s_b | s_i, a^\star_{h,i}) = \frac{1}{H}\paren{\frac{1}{2} - \tau},
        \end{aligned}
      \right.
    \end{equation*}
    where $\tau\le 1/3$ is a parameter to be determined.
  \item At $h\ge H+1$, all bandit states transit to one of $s_g$ and $s_b$ with probability $1/2$ each.
  \item $s_g$ and $s_b$ are absorbing states: $\P_h(s_g | s_g, a)=\P_h(s_b | s_b, a)$ for all $h\in[2H+1]$ and all $a\in[A]$.
  \end{itemize}
\item Initial state distribution is uniform on all bandit states: $S_1\sim\Unif\set{s_i}_{i\in[S]}$.
\item Reward: The bandit states do not receive any reward. The good state and bad state also do not receive any reward at $h\le H+1$. Finally, for $h\ge H+2$, the good state receives reward 1 and the bad state receives reward 0 regardless of the action taken:
  \begin{align*}
    r_h(s_g, a) = 1~~~\textrm{and}~~~r_h(s_b, a) = 0~~~\textrm{for}~H+2\le h\le 2H+1.
  \end{align*}
\end{itemize}

We also let $M_{\bzero}$ denote the ``null'' MDP which has the same construction as the above except that there is no ``special'' actions $a^\star_{h,i}$, that is,
\begin{align*}
  \P_h(s_g | s_i, a) = \P_h(s_b | s_i, a) = \frac{1}{2H}~~~\textrm{for all}~a\in[A].
\end{align*}

\paragraph{Policies $\pi_\star$ and $\mu$}
In MDP $M_{\ba^\star}$, at any bandit state $s_i$ and time step $h\le H$, the optimal action to take is $a^\star_{h,i}$ since it induces a slightly higher probability of transiting to the ``good state'' $s_g$ than all other actions. At all other states or time steps, the action does not affect anything (the transition and reward do not depend on the action), so for example we could take $a=1$. To summarize, the following deterministic policy $\pi_\star$ is an optimal policy for $M_{\ba^\star}$:
\begin{itemize}
\item $\pi_{\star,h}(s_i) = a^\star_{h,i}$ for all $i\in[S]$ and $h\in[H]$.
\item $\pi_{\star,h}(s_i) = 1$ for all $h\ge H+1$;
\item $\pi_{\star, h}(s_g)=\pi_{\star, h}(s_b)=1$ for all $h\in[2H+1]$.
\end{itemize}
We define our reference policy $\mu$ as follows
\begin{itemize}
\item $\mu_h(a|s_i) = \frac{1}{K}\indic{1\le a\le K}$ for all $i\in[S]$, and $h\in[H]$.
\item $\mu_h(1|s)=1$ whenever $h\ge H+1$ or $s\in\set{s_g, s_b}$.
\end{itemize}

The following lemma shows that $\mu$ satisfies $C^\star$-concentrability with respect to $\pi_\star$ as long as all optimal actions $a^\star_{h,i}\in\set{1,\dots,K}$. The proof of this lemma is deferred to Section~\ref{appendix:proof-ref-concentrability-online}.
\begin{lemma}[$\mu$ satisfies single-policy concentrability]
  \label{lemma:ref-concentrability-online}
  For any $\ba^\star\in\set{1,\dots,K}^{HS}$, in the MDP $M_{\ba^\star}$, we have
  \begin{align*}
    \sup_{h,s,a} \frac{d^{\pi_\star}_h(s,a)}{d^\mu_h(s,a)} \le K \le C^\star,
  \end{align*}
  where $\pi_\star$ (the optimal policy for $M_{\ba^\star}$) and $\mu$ are defined as above.
\end{lemma}
Lemma~\ref{lemma:ref-concentrability-online} shows that the following family of problems is indeed a subset of the class~\eqref{equation:online-c-class}:
\begin{align*}
\set{ (M_{\ba^\star}, \mu): \ba^\star \in [K]^{HS} } \subset \mc{M}_{C^\star}.
\end{align*}
We further let $\nu$ denote the uniform (prior) distribution on $[K]^{HS}$, that is, $\nu(\ba^\star=\ba_0)=1/K^{HS}$ for all $\ba_0\in[K]^{HS}$.


\paragraph{$\mu$ is uninformative}
Note that in the above family of problems, the reference policy $\mu$ is the same for all MDPs. Therefore $\mu$ is \emph{uninformative} in the sense that the set of all online finetuning algorithms that utilize $\mu$ in its execution is equivalent to the set of all usual online RL algorithms for this particular class of MDPs (which by definition may utilize the additional knowledge that this class of MDPs has a policy $\mu$ with good concentrability). In the following, we assume $\what{\pi}$ is any online RL algorithm for this class of MDPs.


Without loss of generality, we further restrict attention to algorithms that output a deterministic policy (i.e. $\what{\pi}_h(s)\in[A]$)\footnote{This is because we can replace $\P_{\ba_\star}(\what{\pi}_h(s)\neq a^\star_{h,i})$ in our proof for deterministic policies with $\E_{\ba_\star}\brac{ 1 - \what{\pi}_h(a^\star_{h,i}|s) }$ for stochastic policies, and the proof will follow analogously.}. For any deterministic policy $\what{\pi}$ and any index $\ba_\star$, define the \emph{bandit best-arm identification} loss for any $(h,i)\in[H]\times[S]$ as
\begin{align*}
  \ell_{h,i}(\what{\pi}, \ba_\star) \defeq \P_{\ba_\star}\paren{ \what{\pi}_h(s_i) \neq a^\star_{h,i} },
\end{align*}
and the \emph{total bandit loss} as
\begin{align*}
  L(\what{\pi}, \ba_\star) \defeq \sum_{(h,i)\in[H]\times [S]} \ell_{h,i}(\what{\pi}, \ba_\star) = \E_{\ba_\star}\brac{ \# \set{(h,i): \what{\pi}_h(s_i)\neq a^\star_{h,i}} }.
\end{align*}
The loss $L$ measures the expected number of $(h,i)$ pairs on which the algorithm failed to identify the best arm $a^\star_{h,i}$. A large loss will translate to a high suboptimality bound, as we make precise in the following lemma. The proof can be found in Section~\ref{appendix:proof-bandit-loss-to-suboptimality}.
\begin{lemma}
  \label{lemma:bandit-loss-to-suboptimality}
  For any $\ba^\star\in[A]^{HS}$ and any algorithm outputing a deterministic policy $\what{\pi}$, we have
  \begin{align*}
    \E_{\ba^\star}\brac{V_{1,M_{\ba^\star}}^\star - V_{1, M_{\ba^\star}}^{\what{\pi}}} = \sum_{h=1}^H \sum_{i=1}^S \frac{1}{S}\paren{1 - \frac{1}{H}}^{h-1} \tau \cdot \ell_{h,i}(\what{\pi}, \ba^\star) \ge \frac{\tau }{3S} L(\what{\pi}, \ba^\star).
  \end{align*}
\end{lemma}
With Lemma~\ref{lemma:bandit-loss-to-suboptimality} at hand, establishing lower bound on the expected suboptimality reduces to establishing a lower bound on the total bandit loss $L(\what{\pi}, \ba^\star)$.

\paragraph{Lower bounding the total bandit loss}
Fix any $(h,i)$, and consider the individual bandit loss $\E_{\ba^\star\sim \nu}\brac{\ell_{h,i}(\what{\pi}, \ba^\star)}$ averaged over the prior $\ba^\star\sim \nu$. For any algorithm $\what{\pi}$, we decompose the algorithm into (1) the data collection algorithm $\mc{A}$ from which we collect the observed data $X\sim (\P_{\ba^\star}, \mc{A})$, and (2) the estimator $\what{\pi}_h(s_i)=f(X)$ for some measurable function $f:\mc{X}\to[A]$. We have
\begin{align*}
  \E_{\ba^\star\sim \nu}\brac{\ell_{h,i}(\what{\pi}, \ba^\star)} \ge \inf_{f} \E_{\ba^\star\sim \nu}\E_{X\sim (\P_{\ba^\star}, \mc{A})}\brac{ a^\star_{h,i} \neq f(X) }.
\end{align*}
Fixing $\mc{A}$, the $\inf_f$ is taken at the Bayes estimator of $a^\star_{h,i}$, which is a function of the posterior $a^\star_{h,i}|X$ (see, e.g.~\citep[Theorem 1.1 of Section 4]{lehmann2006theory}).

By construction of our MDP $M_{\ba^\star}$, the prior distribution of $a^\star_{h,i}$ is uniform in $[K]$, and the only data that reveal information about $a^\star_{h,i}$ (i.e. data whose likelihood is affected by $a^\star_{h,i}$) is the observed transitions from state $s_i$ at step $h$ (i.e. $\set{(A_h, S'_{h+1}): S_h = s_i}$). Therefore, the posterior distribution $a^\star_{h,i}|X$ depends only on $\set{(A_h, S'_{h+1}): S_h = s_i}$ as well. Further, a set of sufficient statistics for this posterior is the visitation counts $\set{N _h(s_i, a, s'):a\in[A],s'\in\mc{S}}$, where $N_h(s_i, a, s')$ denote the visitation count of $(s_i, a, s')$ at step $h$ within the data $X$.
Therefore, we can restrict the $\inf_f$ to the inf over functions of the vistation counts only (denoted as $g$), and obtain that
\begin{align*}
  & \quad \E_{\ba^\star\sim\nu}\brac{\ell_{h,i}(\what{\pi}, \ba^\star)} \ge \inf_{g} \E_{\ba^\star\sim \nu} \P_{X\sim (\P_{\ba^\star}, \mc{A})}\brac{a^{\star}_{h,i} \neq g \paren{ \set{N_h(s_i, a, s'):a\in[A], s'\in\mc{S}} } } \\
  & \ge \inf_{g} \E_{\ba^\star\sim \nu} \P_{X\sim (\P_{\bzero}, \mc{A})} \brac{a^{\star}_{h,i} \neq g\paren{ \set{N_h(s_i, a, s'):a\in[A], s'\in\mc{S}} } } \\
  & \qquad - \E_{\ba^\star\sim\nu} \TV\paren{ \set{N_h(s_i, a, s')}|_{\P_{\bzero}, \mc{A}},  \set{N_h(s_i, a, s')}|_{\P_{\ba^\star}, \mc{A}} } \\
  & \stackrel{(i)}{\ge} \inf_{g} \E_{\ba^\star_{-(h,i)}\sim \nu_{-(h,i)}} \frac{1}{K}\sum_{a^\star_{h,i}=1}^K \P_{X\sim (\P_{\bzero}, \mc{A})} \brac{a^{\star}_{h,i} \neq g\paren{ \set{N_h(s_i, a, s'):a\in[A], s'\in\mc{S}} } } \\
  & \qquad - \E_{\ba^\star\sim\nu} \sqrt{\frac{1}{2}\KL\paren{ \set{N_h(s_i, a, s')}|_{\P_{\bzero}, \mc{A}} \| \set{N_h(s_i, a, s')}|_{\P_{\ba^\star}, \mc{A}}} } \\
  & \stackrel{(ii)}{\ge} \frac{K-1}{K} - \E_{\ba^\star\sim\nu} \sqrt{\frac{1}{2}\sum_{a=1}^A \E_{\P_{\bzero},\mc{A}}\brac{N_h(s_i, a)} \cdot \KL\paren{ \P_{0,h}(\cdot|s_i, a) \| \P_{\ba^\star,h}(\cdot|s_i, a)}} \\
  & \stackrel{(iii)}{\ge} \frac{1}{2} - \E_{\ba^\star\sim\nu} \sqrt{\frac{1}{2} \E_{\P_{\bzero}, \mc{A}} \brac{N_h(s_i, a^\star_{h,i})} \cdot \KL\paren{  \P_{0,h}(\cdot|s_i, a^\star_{h,i})  \|  \P_{\ba^\star,h}(\cdot|s_i, a^\star_{h,i})  }} \\
  & \stackrel{(iv)}{\ge} \frac{1}{2}  -\E_{\ba^\star\sim\nu}  \sqrt{\E_{\P_{\bzero}, \mc{A}} \brac{N_h(s_i, a^\star_{h,i})} \cdot 2\tau^2/H } \\
  & \stackrel{(v)}{\ge} \frac{1}{2}  -\sqrt{\frac{1}{K} \sum_{a=1}^K \E_{\P_{\bzero}, \mc{A}} \brac{N_h(s_i, a)} \cdot 2\tau^2/H } \\
  & \ge \frac{1}{2} - \sqrt{\E_{\P_{\bzero}, \mc{A}} \brac{N_h(s_i)} \cdot 2\tau^2/HK}.
\end{align*}
Above, (i) used Pinsker's inequality; (ii) used the KL divergence decomposition ~\citep[Lemma 15.1]{lattimore2020bandit}, and the fact that $\frac{1}{K}\sum_{a^\star_{h,i}=1}^K \P_{X\sim (\P_{\bzero}, \mc{A})} \brac{a^{\star}_{h,i} \neq g\paren{ \set{N_h(s_i, a, s'):a\in[A], s'\in\mc{S}} } }$ is at least $(K-1)/K$ for any $g$ (it equals either $(K-1)/K$ or $1$ depending on whether $g\in[K]$); (iii) used our construction that $\P_{\bzero}$ and $\P_{\ba^\star}$ at step $h$ and state $s_i$ only differ in the arm $a^\star_{h,i}$; (iv) used the bound that
\begin{align*}
  \KL\paren{ \paren{1-\frac{1}{H}, \frac{1}{2H}, \frac{1}{2H}} \bigg\| \paren{1-\frac{1}{H}, \frac{1}{H}\paren{\frac{1}{2} + \tau}, \frac{1}{H}\paren{\frac{1}{2} - \tau}}} = \frac{1}{2H}\log\frac{1}{1-4\tau^2} \le 4\tau^2/H
\end{align*}
for $\tau\in[0, 0.4]$; and finally (v) used Jensen's inequality and the fact that $a^\star_{h,i}\sim\Unif([K])$ under $\nu$.

Summing the preceding bound over all $(h,i)$, we get that for any algorithm $\what{\pi}$,
\begin{align*}
  & \quad \E_{\ba^\star\sim\nu}\brac{ L(\what{\pi}, \ba^\star) } = \sum_{h=1}^H \sum_{i=1}^S \E_{\ba^\star\sim\nu}\brac{\ell_{h,i}(\what{\pi}, \ba^\star)} \\
  & \ge \frac{HS}{2} - \sum_{h=1}^H\sum_{i=1}^S \sqrt{ \E_{\P_{\bzero}, \mc{A}} \brac{N_h(s_i)} \cdot 2\tau^2/HK  } \\
  & \ge HS\Bigg\{ \frac{1}{2} - \sqrt{\frac{1}{HS}\underbrace{\sum_{h=1}^H\sum_{i=1}^S \E_{\P_{\bzero}, \mc{A}} \brac{N_h(s_i)}}_{nH} \cdot 2\tau^2/HK} \Bigg\} = HS\paren{\frac{1}{2} - \sqrt{ 2\tau^2n/HSK }}.
\end{align*}
Therefore, as long as $\sqrt{2\tau^2n/HSK}\le 1/4$, i.e. $n\le HSK/(32\tau^2)$, we have
\begin{align*}
  \E_{\ba^\star\sim\nu}\brac{L(\what{\pi}, \ba^\star)} \ge HS/4.
\end{align*}

\paragraph{Bandit loss to MDP suboptimality loss}
By the above lower bound on the average risk, for any algorithm $\what{\pi}$, there must exist some instance $\ba^\star\in[K]^{HS}$ for which $L(\what{\pi}, \ba^\star)\ge HS/4$. On this $\ba^\star$, by Lemma~\ref{lemma:bandit-loss-to-suboptimality}, we get
\begin{align*}
  \E_{\ba^\star}\brac{V_{1,M_{\ba^\star}}^\star - V_{1,M_{\ba^\star}}^{\what{\pi}}} \ge \frac{\tau}{3S} \cdot L(\what{\pi}, \ba^\star) \ge \tau H/12.
\end{align*}
Finally, for any $\eps\le 1/12$, take $\tau = 12\eps/H\le 1/3$, we have the following: as long as
\begin{align*}
  n \le \frac{HSK}{32\tau^2} = \frac{H^3SK}{32\cdot 12^2\eps^2},
\end{align*}
(which is satisfied if $n\le c_0\cdot H^3S\min\set{C^\star, A}/\eps^2$ for some absolute $c_0>0$), we have $\E_{\ba^\star}\brac{V_{1,M_{\ba^\star}}^\star - V_{1,M_{\ba^\star}}^{\what{\pi}}} \ge \eps$. This is the desired result.
\qed

\subsection{Proof of Lemma~\ref{lemma:ref-concentrability-online}}
\label{appendix:proof-ref-concentrability-online}
Fix any $\ba^\star$, we show that $d^{\pi_\star}_h(s,a)/d^\mu_h(s,a)\le K$ for all $(s,a)$. Note that it suffices to consider actions taken by $\pi_\star$ only.

\paragraph{Bandit states}
For any bandit state $s\in\set{s_i}_{i\in[S]}$ and any $h\le H$, by construction of our MDP, we have $d^{\pi_\star}_h(s_i) = d^\mu_h(s_i) = \paren{1 - \frac{1}{H}}^{h-1}$, $\pi_{\star, h}(a^\star_{h,i} | s_i)=1$, and $\mu_h(a^\star_{h,i} | s_i) = 1/K$, therefore
\begin{align*}
  \frac{d^{\pi_\star}_h(s_i, a^\star_{h,i})}{d^\mu_h(s_i, a^\star_{h,i})} = \frac{d^{\pi_\star}_h(s_i) \cdot \pi_{\star, h}(a^\star_{h,i}|s_i)}{d^{\mu}_h(s_i) \cdot \mu_{h}(a^\star_{h,i}|s_i)} = K.
\end{align*}
At $h= H+1$, we have $d^{\pi_\star}_h(s_i)=d^\mu_h(s_i)=(1-1/H)^H$, and both $\pi_\star$ and $\mu$ takes action $1$ deterministically, and thus $d^{\pi_\star}_h(s_i, 1)/d^\mu_h(s_i, 1)=1$. At $h\ge H+2$, we have $d^{\pi_\star}_h(s_i)=d^\mu_h(s_i)=0$. This verifies the $K$-concentrability for all $s_i$.

\paragraph{Good state and bad state}
For the good state $s_g$ and bad state $s_b$, since both $\pi_\star$ and $\mu$ takes deterministic action $1$ at all $h$, it suffices to bound distribution ratio $d^{\pi_\star}_h(s_{\{g, b\}})/d^\mu_h(s_{\{g, b\}})$ over the states only (instead of joint state-actions).

Recall that $\pi_\star$ always takes the optimal action $a^\star_{h,i}$ which leads to $1/H \cdot (1/2+\tau)$ transition probability to the ``good state'' $s_g$. Thus at any $h\le H+1$, we have
\begin{align*}
  d^{\pi_\star}_h(s_g) = \sum_{h'=1}^{h-1}\sum_{i=1}^S \frac{1}{S}\paren{1 - \frac{1}{H}}^{h'-1} \cdot \frac{1}{H}\paren{\frac{1}{2} + \tau} = \paren{\sum_{h'=1}^{h-1} \paren{1 - \frac{1}{H}}^{h'-1} \cdot \frac{1}{H}} \cdot \paren{\frac{1}{2} + \tau}.
\end{align*}
In contrast, $\mu$ only takes the optimal action $a^\star_{h,i}$ with probability $1/K$, and thus we have
\begin{align*}
  d^\mu_h(s_g) = \paren{\sum_{h'=1}^{h-1} \paren{1 - \frac{1}{H}}^{h'-1} \cdot \frac{1}{H}} \cdot \paren{\frac{1}{2} + \frac{\tau}{K}}.
\end{align*}
Therefore
\begin{align*}
  \frac{d^{\pi_\star}_h(s_g) }{d^{\mu}_h(s_g) } = \frac{1/2+\tau}{1/2+\tau/K} \le \frac{1/2+\tau}{1/2} \le 2\le K.
\end{align*}
Similarly, for the ``bad state'' $s_b$ we have
\begin{align*}
  \frac{d^{\pi_\star}_h(s_b) }{d^{\mu}_h(s_b) } = \frac{1/2-\tau}{1/2-\tau/K} \le 1\le K.
\end{align*}
For $h\ge H+2$, we have
\begin{align*}
  \frac{d^{\pi_\star}_h(s_g)}{d^\mu_h(s_g)} = \frac{d^{\pi_\star}_{H+1}(s_g) + \frac{1}{2} \cdot (1-1/H)^H}{d^\mu_{H+1}(s_g) + \frac{1}{2}\cdot (1-1/H)^H} \le 2 \le K,
\end{align*}
and similarly
\begin{align*}
  \frac{d^{\pi_\star}_h(s_b)}{d^\mu_h(s_b)} = \frac{d^{\pi_\star}_{H+1}(s_b) + \frac{1}{2} \cdot (1-1/H)^H}{d^\mu_{H+1}(s_b) + \frac{1}{2}\cdot (1-1/H)^H} \le 1 \le K,
\end{align*}
This verifies the $K$-concentrability for $(s_g, s_b)$ as well.
\qed

\subsection{Proof of Lemma~\ref{lemma:bandit-loss-to-suboptimality}}
\label{appendix:proof-bandit-loss-to-suboptimality}
Fix any $\ba^\star\in[A]^{HS}$, by construction of our MDP $M_{\ba^\star}$, only the good state $s_g$ receives a $+1$ reward starting at $h\in\set{H+2,\dots,2H+1}$. Along each trajectory, there will be exactly one transition from the bandit states $\set{s_i}$ to one of $(s_g, s_b)$. This transition can happen at step $h\le H$ and state $s_i$ with probability $1/H\cdot (1/2+\tau)$ if the optimal action $a^\star_{h,i}$ is taken, or probability $1/(2H)$ if any other action is taken. The transition can also happen at step $h=H+1$ but with the same transition probability regardless of the policy. Further, note that the state distribution $d^\pi_h(s_i)=1/S\cdot (1-1/H)^{h-1}\eqdef d_h(s_i)$ (for $h\le H$) does not depend on the policy $\pi$. Therefore, we have
\begin{align*}
  & \quad V_{1, M_{\ba^\star}}^\star - V_{1, M_{\ba^\star}}^{\what{\pi}} \\
  & = \sum_{h=1}^H \sum_{i=1}^S d_h(s_i) \cdot \brac{ \frac{1}{H}\paren{\frac{1}{2}+\tau} - \frac{1}{2H}} \cdot \indic{\what{\pi}_h(s_i)\neq a^\star_{h,i}} \cdot H \\
  & = \sum_{h=1}^H \sum_{i=1}^S \frac{1}{S}\paren{1 - \frac{1}{H}}^{h-1} \tau \cdot \indic{\what{\pi}_h(s_i)\neq a^\star_{h,i}}.
\end{align*}
Taking expectation with respect to the algorithm execution within the MDP $M_{\ba^\star}$, we get
\begin{align*}
  & \quad \E_{\ba^\star}\brac{  V_{1, M_{\ba^\star}}^\star - V_{1, M_{\ba^\star}}^{\what{\pi}} } = \sum_{h=1}^H \sum_{i=1}^S \frac{1}{S}\paren{1 - \frac{1}{H}}^{h-1}\tau \cdot \underbrace{\P_{\ba^\star}\paren{\what{\pi}_h(s_i)\neq a^\star_{h,i}}}_{\ell_{h,i}(\what{\pi}, \ba^\star)} \\
  & = \sum_{h=1}^H \sum_{i=1}^S \frac{1}{S}\underbrace{\paren{1 - \frac{1}{H}}^{h-1}}_{\ge (1-1/H)^{H-1}\ge e^{-1}\ge 1/3}\tau \cdot \ell_{h,i}(\what{\pi}, \ba^\star) \\
  & \ge \frac{\tau}{3S}\cdot \sum_{h=1}^H \sum_{i=1}^S \ell_{h,i}(\what{\pi}, \ba^\star) = \frac{\tau}{3S} \cdot L(\what{\pi}, \ba^\star).
\end{align*}
This proves the lemma.
\qed
\section{Proofs for Section~\ref{section:online-upper}}
\label{appendix:proof-online-upper}

\subsection{Algorithm \ucbviuplow}

\begin{algorithm}[t]
  \setstretch{1.1}
  \caption{UCB-VI with Upper and Lower Confidence Bounds (\ucbviuplow)}
  \label{algorithm:ucbviuplow}
  \begin{algorithmic}[1]
    \REQUIRE Starting time step $h_\star+1$, end time step $H$, number of episodes $n_{\rm UCB}$.
    \STATE {\bf Initialize:} For any $(s, a, h,s')$: $\up{Q}_h(s, a) \setto H-h_\star$, $\low{Q}_h(s, a)\setto 0$, $N_h(s)=N_h(s, a)=N_h(s,a,s')\setto 0$.
    \FOR{Episode $k=1,\dots,n_{\rm UCB}$}
    \FOR{step $h=H,\dots,h_\star+1$}
    \FOR{$(s,a)\in \mc{S}\times\mc{A}$}
    \STATE $t\setto N_h(s, a)$.
    \IF{$t>0$}
    \STATE $\beta\setto \Bonus(t, \hat{\V}_h[(\up{V}_{h+1} + \low{V}_{h+1})/2](s, a))$~\label{line:ucbviuplow-bonus}
    \STATE $\gamma\setto c/(H-h_\star) \cdot \hat{\P}_h(\up{V}_{h+1} - \low{V}_{h+1})(s, a)$.
    \STATE $\up{Q}_h(s,a)\setto \min\set{(r_h+\hat{\P}_h\up{V}_{h+1})(s, a) + \gamma + \beta, H-h_\star}$.
    \STATE $\low{Q}_h(s,a)\setto \max\set{(r_h+\hat{\P}_h\low{V}_{h+1})(s, a) - \gamma - \beta, 0}$    
    \ENDIF
    \ENDFOR
    \FOR{$s\in\mc{S}$}
    \STATE $\pi_h(s)\setto \argmax_{a\in\mc{A}} \up{Q}_h(s, a)$.
    \STATE $\up{V}_h(s)\setto \up{Q}_h(s, \pi_h(s))$; $\low{V}_h(s)\setto \low{Q}_h(s, \pi_h(s))$.
    \ENDFOR
    \ENDFOR
    \STATE Receive an initial state $s_{h_\star+1}$ from the MDP.
    \FOR{step $h=h_\star+1,\dots,H$}
    \STATE Take action $a_h= \pi_h(s_h)$, observe reward $r_h$ and next state $s_{h+1}$.
    \STATE Add 1 to $N_h(s_h)$, $N_h(s_h, a_h)$, and $N_h(s_h, a_h, s_{h+1})$. 
    \STATE $\hat{\P}_h(\cdot|s_h, a_h)\setto N_h(s_h, a_h, \cdot) / N_h(s_h, a_h)$. 
    \ENDFOR
    \ENDFOR
    \STATE Let $(\up{V}_h^k, \low{V}_h^k, \pi^k)$ denote the value estimates and policy at the beginning of episode $k$.
    \FOR{$s\in\mc{S}$}
    \STATE $\up{V}^{\rm out}_{h_\star+1}(s)\setto \frac{1}{N_{h_\star+1}(s)} \sum_{k:s_{h_\star+1}^k=s} \up{V}_{h_\star+1}^k(s)$. \label{line:vup-out}
    \STATE $\low{V}^{\rm out}_{h_\star+1}(s)\setto \frac{1}{N_{h_\star+1}(s)} \sum_{k:s_{h_\star+1}^k=s} \low{V}_{h_\star+1}^k(s)$. \label{line:vlow-out}
    \STATE Let policy $\pi^{(s)}_{h_\star+1:H}$ be the uniform mixture of $\pi^k_{h_\star+1:H}$ over all $\set{k:s_{h_\star+1}^k=s}$.
    \ENDFOR
    \RETURN Value estimates $(\up{V}^{\rm out}_{h_\star+1}, \low{V}^{\rm out}_{h_\star+1})$. \\
    \hspace{3em} Mixture policy $\pi^{\rm out}_{h_\star+1:H}$ defined as follows: Play policy $\pi^{(s)}_{h_\star+1:H}$ if $s_{h_\star+1}=s$. \label{line:pi-out}
  \end{algorithmic}
\end{algorithm}

We present the \ucbviuplow~algorithm in Algorithm~\ref{algorithm:ucbviuplow}. This algorithm is a specialization of the {\sc Nash-VI} algorithm of~\citep{liu2020sharp} (originally developed for two-player Markov games) into the case with a single player, and with additional modifications on the output value functions and policy that are similar to the certified policy technique of~\citep{bai2020near}. Above, the $\Bonus$ function on Line~\ref{line:ucbviuplow-bonus} is taken as the Bernstein bonus:
\begin{align*}
  \Bonus(t, \hat{\sigma}^2) \defeq c\paren{\sqrt{\hat{\sigma}^2\iota/t} + (H-h_\star)^2S\iota/t},
\end{align*}
where $\iota\defeq \log(HSAn/\delta)$ is a log factor and $c>0$ is some absolute constant.

We now present the main guarantee for \ucbviuplow, which is going to be used in proving the main theorem.
\begin{lemma}[Theoretical guarantee for \ucbviuplow]
  \label{lemma:ucbviuplow}
  Suppose Algorithm~\ref{algorithm:ucbviuplow} is run for $n_{\rm UCB}$ episodes. Then, the output lower value estimate $\low{V}^{\rm out}_{h_\star+1}$ and the policy $\pi^{\rm out}_{h_\star+1:H}$ satisfies the following with probability at least $1-\delta$:
  \begin{enumerate}[label=(\alph*),wide]
  \item Small error in lower value estimates: $0\le \low{V}^{\rm out}_{h_\star+1}(s)\le V^\star_{h_\star+1}(s)$ for all $s\in\mc{S}$, and
    \begin{align*}
      \sum_{s\in\mc{S}} d_{h_\star+1}(s) \paren{V^\star_{h_\star+1}(s) - \low{V}^{\rm out}_{h_\star+1}(s)} \le O\paren{ \sqrt{\frac{(H-h_\star)^3SA\iota^3}{n_{\rm UCB}}} + \frac{(H-h_\star)^3S^2A\iota^3}{n_{\rm UCB}}},
    \end{align*}
    where $d_{h_\star+1}(\cdot)$ is the distribution of the initial state $s_{h_\star+1}$.
  \item $\pi^{\rm out}$ achieves at least $\low{V}^{\rm out}$ value: we have $V_{h_\star+1}^{\pi^{\rm out}}(s) \ge \low{V}^{\rm out}_{h_\star+1}(s)$ for all $s\in\mc{S}$.
  \end{enumerate}
\end{lemma}
\begin{proof}
  First notice that Algorithm~\ref{algorithm:ucbviuplow} is a special case of the \textsc{Nash-VI} algorithm for two-player Markov games~\citep[Algorithm 1]{liu2020sharp}, where the number of actions for the min-player is one so that the Markov game reduces to a single-player MDP, and the game starts at step $h_\star+1$ (so that the horizon length is $H-h_\star$ instead of $H$).

  Therefore, by~\citep[Theorem 4]{liu2020sharp}, with probability at least $1-\delta/2$ the following happens
  \begin{align*}
    \sum_{k=1}^{n_{\rm UCB}} \up{V}_{h_\star+1}^k(s_{h_\star+1}^k) - \low{V}_{h_\star+1}^k(s_{h_\star+1}^k) \le \sqrt{(H-h_\star)^3SAn_{\rm UCB}\iota} + (H-h_\star)^3S^2A\iota^2,
  \end{align*}
  where $\iota\defeq \log(HSA/\delta)$. Further,~\citep[Lemma 22]{liu2020sharp} implies that for any $(s,k)$,
  \begin{align*}
    \up{V}_{h_\star+1}^k(s) \ge V_{h_\star+1}^\star(s) \ge V_{h_\star+1}^{\pi^k}(s) \ge \low{V}_{h_\star+1}^k(s)
  \end{align*}
  on the same good event. Using the relation betwen $\up{V}^k_{h_\star+1}$ and $V_{h_\star+1}^\star$ yields
    \begin{align*}
      & \quad \sum_{s\in\mc{S}} N_{h_\star+1}(s) \paren{V^\star_{h_\star+1}(s) - \low{V}^{\rm out}_{h_\star+1}(s)} \\
      & \stackrel{(i)}{=} \sum_{k=1}^{n_{\rm UCB}} V^\star_{h_\star+1}(s_{h_\star+1}^k) - \low{V}_{h_\star+1}^k(s_{h_\star+1}^k) \le \sqrt{(H-h_\star)^3SAn_{\rm UCB}\iota} + (H-h_\star)^3S^2A\iota^2,
  \end{align*}
  where (i) used the definition of our $\low{V}^{\rm out}$ in Line~\ref{line:vlow-out}. Now, since $N_{h_\star+1}(s) \sim \Bin(n_{\rm UCB}, d_{h_\star+1}(s))$, by Lemma~\ref{lemma:binomial-concentration} and a union bound over $s\in\mc{S}$, we have with probability at least $1-\delta/2$ that
  \begin{align*}
    N_{h_\star+1}(s)\vee 1 \ge \frac{1}{C\iota} \cdot n_{\rm UCB}d_{h_\star+1}(s)
  \end{align*}
  simultaneously for all $s\in\mc{S}$. 
  
  Plugging this into the preceding bound further yields
  \begin{align*}
    & \quad \sum_{s\in\mc{S}} d_{h_\star+1}(s) \paren{V^\star_{h_\star+1}(s) - \low{V}^{\rm out}_{h_\star+1}(s)} \\
    & \le \frac{C\iota}{n_{\rm UCB}} \sum_{s\in\mc{S}} \paren{N_{h_\star+1}(s)\vee 1} \cdot \paren{V^\star_{h_\star+1}(s) - \low{V}^{\rm out}_{h_\star+1}(s)} \\
    & \le \frac{C\iota}{n_{\rm UCB}} \sum_{s\in\mc{S}} N_{h_\star+1}(s) \cdot \paren{V^\star_{h_\star+1}(s) - \low{V}^{\rm out}_{h_\star+1}(s)} + \frac{C\iota}{n_{\rm UCB}} \sum_{s\in\mc{S}} \indic{N_{h_\star+1}(s) = 0} \cdot \underbrace{\paren{V^\star_{h_\star+1}(s) - \low{V}^{\rm out}_{h_\star+1}(s)}}_{\le H-h_\star} \\
    & \le C\sqrt{\frac{(H-h_\star)^3SA\iota^3}{n_{\rm UCB}}} + C\cdot\frac{(H-h_\star)^3S^2A\iota^3 + (H-h_\star)S\iota}{n_{\rm UCB}} \\
    & \le C\paren{ \sqrt{\frac{(H-h_\star)^3SA\iota^3}{n_{\rm UCB}}} + \frac{(H-h_\star)^3S^2A\iota^3}{n_{\rm UCB}} }
  \end{align*}
  This shows part (a).
  
  For part (b), by our definition of $\pi^{\rm out}$ (Line~\ref{line:pi-out}), we have for any $s\in\mc{S}$ that
  \begin{align*}
    & \quad V^{\pi^{\rm out}}_{h_\star+1}(s) = V^{\pi^{(s)}}_{h_\star+1}(s) \\
    & = \frac{1}{N_{h_\star+1}(s)} \sum_{k:s_{h_\star+1}^k=s} V_{h_\star+1}^{\pi^k}(s) \\
    & \ge \frac{1}{N_{h_\star+1}(s)} \sum_{k:s_{h_\star+1}^k=s} \low{V}_{h_\star+1}^k(s) = \low{V}^{\rm out}_{h_\star+1}(s).
  \end{align*}
  This shows part (b). 
\end{proof}

\subsection{Algorithm \truncpevi}

The algorithm \truncpevi~is similar as the \pevi~algorithm (Algorithm~\ref{algorithm:pevi-adv}) except that the algorithm uses a plug-in estimate of $V^\star_{h_\star+1}$, and only performs (pessimistic) value iteration within step $1$ to $h_\star$. For completeness, we describe the algorithm in Algorithm~\ref{algorithm:trunc-pevi}.

\begin{algorithm}[t]
  \caption{Truncated-PEVI-ADV}
  \label{algorithm:trunc-pevi}
  \begin{algorithmic}[1]
    \REQUIRE Offline dataset $\mc{D}=\set{(s_1^{(i)}, a_1^{(i)}, r_1^{(i)}, \dots,s_H^{(i)}, a_H^{(i)}, r_H^{(i)})}_{i=1}^{n_{\rm offline}}$. End time step $h_\star$. Value function $V^{\rm init}_{h_\star+1}$.
    \STATE Split the dataset $\mc{D}$ into $\mc{D}_{\rf}$, $\mc{D}_{0}$ and $\set{\mc{D}_{h,1}}_{h=1}^{h_\star}$ uniformly at random:
    \begin{align*}
      & n_{\rf} \coloneqq \abs{\mc{D}_{\rf}} = n_{\rm offline}/3, ~~ n_0 \coloneqq \abs{\mc{D}_0} = n_{\rm offline}/3, ~~ n_{1,h} \coloneqq \abs{\mc{D}_{h,1}} \defeq n_{\rm offline}/(3h_\star).
    \end{align*}
    \STATE Learn a reference value function $\what{V}^{\rf}_{1:h_\star} \setto \vilcb(\mc{D}_{\rf})$ via of {\bf a truncated version of} \vilcb (Algorithm~\ref{algorithm:vilcb}), with the modification that only the datae from steps $1:h_\star$ are used, and {\bf $\hat{V}^{\rf}_{h_\star+1}\setto V^{\rm init}_{h_\star+1}$ and is not updated}.
    \STATE Let $N_{h,0}(s,a)$ and $N_{h,0}(s,a,s')$ denote the visitation count of $(s,a)$ and $(s,a,s')$ at step $h$ within dataset $\mc{D}_0$. Construct empirical model estimates:
    \begin{align*}
      & \what{r}_{h,0}(s,a) \setto r_h(s,a)\indic{N_{h,0}(s,a)\ge 1}, \\
      & \what{\P}_{h,0}(s' | s,a) \setto \frac{N_{h,0}(s,a,s')}{N_{h,0}(s,a)\vee 1}.
    \end{align*}
    Similarly define $N_{h,1}(s,a)$, $N_{h,1}(s,a,s')$, $(\what{r}_{h,1},\what{\P}_{h,1})$ for all $h\in[h_\star]$ based on dataset $\mc{D}_{h,1}$.
    \STATE For all $(h,s,a)$, set
      $ b_{h,0}(s,a) \setto c \cdot \paren{\sqrt{\frac{ [\hat{\V}_{h,0} \hat{V}^{\rf}_{h+1}](s,a) \iota}{N_{h,0}(s,a)\vee 1}} +
        \frac{H \iota}{N_{h,0}(s,a)\vee 1}}$,
    where $\iota\defeq \log(HSA/\delta)$. 
    \STATE Set $\hat{V}_{h_\star+1}(s)\setto V^{\rm init}_{h_\star+1}(s)$ for all $s\in\mc{S}$. (Note that $\hat{V}_{h_\star+1}(s)$ {\bf is not updated in the following}.)
    \FOR{$h=h_\star,\dots,1$}
    \STATE Set
    $
      b_{h,1}(s,a) \setto c \cdot \paren{\sqrt{\frac{ [\hat{\V}_{h,1} (\hat{V}_{h+1} - \hat{V}^{\rf}_{h+1})](s,a) \iota}{N_{h,1}(s,a)\vee 1}} + \frac{H \iota}{N_{h,1}(s,a)\vee 1}}$.
    \STATE Perform value update for all $(s,a)$:
    \begin{align*}
        & \hat{Q}_h(s,a) \setto \what{r}_{h,0}(s,a) + \brac{\what{\P}_{h,0} \hat{V}^{\rf}_{h+1}}(s,a) - b_{h,0}(s,a) + \brac{\what{\P}_{h,1} (\hat{V}_{h+1} - \hat{V}^{\rf}_{h+1})}(s,a) - b_{h,1}(s,a); \\
        & \hat{V}_h(s) \setto \brac{\max_a \hat{Q}_h(s, a)} \vee 0.
    \end{align*}
    \STATE Set $\what{\pi}_h(s)\setto \argmax_a \hat{Q}_h(s,a)$ for all $s\in\mc{S}$.
    \ENDFOR
    \RETURN Policy $\what{\pi}=\set{\what{\pi}_h}_{1\le h\le h_\star}$.
  \end{algorithmic}
\end{algorithm}

\subsection{Proof of Theorem~\ref{theorem:online-upper}}
\label{appendix:proof-online-upper-mainthm}
We are now ready to present the proof of Theorem~\ref{theorem:online-upper}.
Throughout this proof, $C$ denotes an absolute constant that can vary from line to line, and all ``good events'' happen with probability at least $1-\delta/10$, which combine to yield the $1-\delta$ high probability guarantee for the final bound.


\paragraph{Guarantees for learned values}
First, Stage 1 in our Algorithm~\ref{algorithm:hybrid} runs the \ucbviuplow~algorithm with $n_{\rm UCB}=n/2$ episodes with initial state $s_{h_\star+1}\sim d^\mu_{h_\star+1}$. Therefore by Lemma~\ref{lemma:ucbviuplow}, its output lower value estimate $\low{V}_{h_\star+1}$ satisfies $\low{V}_{h_\star+1}(s)\le V^ \star_{h_\star+1}(s)$ for all $s\in\mc{S}$, and
\begin{align*}
  \sum_{s\in\mc{S}} d^{\mu}_{h_\star+1}(s) \paren{ V^ \star_{h_\star+1}(s) - \low{V}_{h_\star+1}(s) } \le C\paren{ \sqrt{\frac{(H-h_\star)^3SA\iota^3}{n}} + \frac{(H-h_\star)^3S^2A\iota^3}{n} }.
\end{align*}
Since $\mu$ satisfies $C^{\rm partial}$ partial concentratbility for steps $1:h_\star$ (by Assumption~\ref{assumption:partial-c}), it also satisfies the \emph{state-wise} concentrability at step $h_\star+1$: For any $s'\in\mc{S}$ we have
\begin{align*}
  \frac{d^{\pi_\star}_{h_\star+1}(s')}{d^\mu_{h_\star+1}(s')} = \frac{\sum_{s,a} d^{\pi_\star}_{h_\star}(s,a) \P_{h_\star}(s'|s, a)}{\sum_{s,a} d^{\mu}_{h_\star}(s,a) \P_{h_\star}(s'|s, a)} \le C^{\rm partial}.
\end{align*}
Applying this in the preceding bound, we get
\begin{align*}
  \sum_{s\in\mc{S}} d^{\pi_\star}_{h_\star+1}(s) \paren{ V^ \star_{h_\star+1}(s) - \low{V}_{h_\star+1}(s) } \le C \cdot C^{\rm partial}\paren{ \sqrt{\frac{(H-h_\star)^3SA\iota^3}{n}} + \frac{(H-h_\star)^3S^2A\iota^3}{n} } \eqdef \eps_0.
\end{align*}

Next, in stage 2 we run the \truncpevi~algorithm. In its first (sub)-stage, we learn the reference value function $\what{V}^{\rf}$ via a truncated version of the \vilcb~algorithm (which runs for $(n-n_{\rm UCB})/3=n/6$ episodes). Since we set $\what{V}^{\rf}_{h_\star+1}\setto \low{V}_{h_\star+1}$ and do not update it, we can imitate the proof of Theorem~\ref{theorem:vilcb} for steps $1$ to $h_\star$ (which uses the $C^{\rm partial}$ partial-concentrability assumed in Assumption~\ref{assumption:partial-c}) and obtain $\what{V}^{\rf}_h(s)\le V^\star_h(s)$ for all $h\in[h_\star]$, $s\in\mc{S}$, and (by replacing $H$ with $h_\star$ at appropriate places in the proof in Section~\ref{appendix:proof-vilcb-main})
\begin{align*}
  & \quad \max_{1\le h\le h_\star} \sum_{s\in\mc{S}} d^{\pi_\star}_h(s) \paren{V^\star_h(s) - \what{V}^{\rf}_h(s)} \\
  & \le \sum_{h=1}^{h_\star} \sum_{s,a} d^{\pi_\star}_h(s,a) \cdot  b_h(s,a) + \sum_{s\in\mc{S}} d^{\pi_\star}_{h_\star+1}(s) \paren{ V^ \star_{h_\star+1}(s) - \low{V}_{h_\star+1}(s) } \\
  & \le C\sqrt{\frac{H^2h_\star^3SC^{\rm partial}\iota^2}{n}} + \sum_{s\in\mc{S}} d^{\pi_\star}_{h_\star+1}(s) \paren{ V^ \star_{h_\star+1}(s) - \low{V}_{h_\star+1}(s) } \\
  & \le C\sqrt{\frac{H^2h_\star^3SC^{\rm partial}\iota^2}{n}} + \eps_0 \eqdef \eps_1.
\end{align*}
In its second sub-stage, we learn the final value function $\what{V}$ in a similar fashion as the reference-advantage updates in Algorithm~\ref{algorithm:pevi-adv} (which runs for $n_0+n_1$ episodes where $n_0=n_1=n/6$). Notice again we set $\hat{V}_{h_\star+1}\setto \low{V}_{h_\star+1}$ and do not update it. Therefore we can imiate the proof of Theorem~\ref{theorem:main} (again, this uses the $C^{\rm partial}$ partial-concentrability assumed in Assumption~\ref{assumption:partial-c}) to obtain $\what{V}_h(s)\le V^\star_h(s)$ for all $h\in[h_\star]$, $s\in\mc{S}$ and (by replacing $H$ with $h_\star$ at appropriate places in Section~\ref{sec:provemainstep3})
\begin{align*}
  & \quad \max_{1\le h\le h_\star} \sum_{s\in\mc{S}} d^{\pi_\star}_h(s) \paren{V^\star_h(s) - \what{V}_h(s)} \\
  & \le \sum_{h=1}^{h_\star} \sum_{s,a} d^{\pi_\star}_h(s,a) \cdot  \paren{b_{h,0}(s,a) + b_{h,1}(s,a)} + \sum_{s\in\mc{S}} d^{\pi_\star}_{h_\star+1}(s) \paren{ V^ \star_{h_\star+1}(s) - \low{V}_{h_\star+1}(s) } \\
  & \le C\bigg(\sqrt{\frac{H^2h_\star SC^{\rm partial}\iota^2}{n}} + \frac{H^2h_\star SC^{\rm partial}\iota^{3/2} + Hh_\star^3 SC^{\rm partial}\iota^2}{n} \\
  & \qquad \quad + \sqrt{\frac{h_\star^4SC^{\rm partial}\iota^2}{n}} \cdot H \cdot \eps_1 + \frac{H^2h_\star^2 SC^{\rm partial}\iota^{3/2} + Hh_\star^2SC^{\rm partial}\iota^2}{n} \bigg) + \eps_0 \\
  & \le C\bigg(\sqrt{\frac{H^2h_\star SC^{\rm partial}\iota^2}{n}} + \frac{H^2h_\star SC^{\rm partial}\iota^{3/2} + Hh_\star^3 SC^{\rm partial}\iota^2}{n} \\
  & \qquad \quad + \sqrt{\frac{H^2h_\star^4SC^{\rm partial}\iota^2}{n}} \cdot \paren{\sqrt{\frac{H^2h_\star^3SC^{\rm partial}\iota^2}{n}} + \eps_0}+ \frac{H^2h_\star^2 SC^{\rm partial}\iota^{3/2} + Hh_\star^2SC^{\rm partial}\iota^2}{n} \bigg) + \eps_0 \\
  & \eqdef \eps_2.
\end{align*}
We first let $\eps_0\le \eps/2$ which requires
\begin{align*}
  n\ge N_0 \defeq C\brac{(H-h_\star)^3SA\iota^3(C^{\rm partial})^2 + (H-h_\star)^3S^2A\iota^3/\eps}.
\end{align*}
Then to let the rest of the terms above be also bounded by $\eps/2$, a sufficient condition is
\begin{align*}
  n\ge N_2 \defeq C\brac{\frac{H^2h_\star SC^{\rm partial}\iota^2}{\eps^2} + \frac{H^2h_\star^{3.5}SC^{\rm partial}\iota^2}{\eps} + H^2h_\star^4SC^{\rm partial}\iota^2}.
\end{align*}
Combined, this shows that $\max_{1\le h\le h_\star} \sum_{s\in\mc{S}} d^{\pi_\star}_h(s) \paren{V^\star_h(s) - \what{V}_h(s)}\le \eps$ if
\begin{align*}
  & \quad n\ge C\bigg( \frac{H^2h_\star SC^{\rm partial}\iota^2 + (H-h_\star)^3SA(C^{\rm partial})^2\iota^3}{\eps^2} \\
  & \qquad + \frac{H^2h_\star^{3.5}SC^{\rm partial}\iota^2 + (H-h_\star)^3S^2AC^{\rm partial}\iota^3}{\eps} + H^2h_\star^4SC^{\rm partial}\iota^2 \bigg) \\
  & \ge \max\set{N_0, N_2}.
\end{align*}
Further, when $\eps \le \min\set{h_\star^{-2.5}, C^{\rm partial}/S}$, the $\eps^{-2}$ term above dominates and thus a sufficient condition for the above is
\begin{align}
  \label{equation:n-requirement}
  n \ge \wt{O}\paren{ \frac{H^2h_\star SC^{\rm partial} + (H-h_\star)^3SA(C^{\rm partial})^2}{\eps^2} }.
\end{align}

\paragraph{Guarantees for output policy}
We also show that the final output policy $\what{\pi}$ of Algorithm~\ref{algorithm:hybrid} also satisfies $V_1^\star(s_1) - V_1^{\what{\pi}}(s_1)\le\eps$ building on the above guarantee for $\what{V}$. First, at step $h_\star+1$, as $\what{\pi}_{(h_\star+1):H}=\what{\pi}^{\rm UCB}_{(h_\star+1):H}$, by Lemma~\ref{lemma:ucbviuplow}(b) we have for all $s\in\mc{S}$ that
\begin{align*}
  V^{\what{\pi}}_{h_\star+1}(s) \ge \low{V}_{h_\star+1}(s) = \hat{V}_{h_\star+1}(s).
\end{align*}
Second, using this as a base step for the induction argument in~\ref{lem:algo1monot}, we get that $V^{\what{\pi}}_h(s) \ge \hat{V}_h(s)$ for all $h\in[h_\star]$ and $s$. In particular, at $h=1$ we have $V^{\what{\pi}}_1(s_1) \ge \hat{V}_1(s_1)$. Therefore
\begin{align*}
  V_1^\star(s_1) - V^{\what{\pi}}_1(s_1) \le V_1^\star(s_1) - \hat{V}_1(s_1) = \sum_{s\in\mc{S}} d^{\pi_\star}(s) \paren{V_1^\star(s) - \hat{V}_1(s)} \le \eps.
\end{align*}
This shows the desired near-optimality guarantee for $\what{\pi}$ whenever $\eps \le \min\set{h_\star^{-2.5}, C^{\rm partial}/S}$ and the number of episodes $n$ satisfies~\eqref{equation:n-requirement}. This proves Theorem~\ref{theorem:online-upper}.
\qed

\end{document}